\PassOptionsToPackage{numbers}{natbib}
\def\arxivversion{1}
\documentclass{article}

\usepackage[main,final]{neurips_2026}

\usepackage[utf8]{inputenc}
\usepackage[T1]{fontenc}
\usepackage{hyperref}
\hypersetup{hidelinks}
\usepackage{url}
\usepackage{booktabs}
\usepackage{amsfonts}
\usepackage{amsmath}
\usepackage{amssymb}
\usepackage{mathtools}
\usepackage{microtype}
\usepackage{multirow}
\usepackage{xcolor}
\usepackage{enumitem}
\usepackage{amsthm}
\usepackage{array}
\usepackage{algorithm}
\usepackage{algpseudocode}
\usepackage{wrapfig}
\usepackage{graphicx}
\usepackage{float}
\usepackage{placeins}

\newcommand{\E}{\mathbb{E}}
\newcommand{\R}{\mathbb{R}}
\newcommand{\Ncal}{\mathcal{N}}
\newcommand{\Ycal}{\mathcal{Y}}
\newcommand{\norm}[1]{\left\lVert #1 \right\rVert}
\newcommand{\inner}[2]{\left\langle #1,#2 \right\rangle}
\newcommand{\lin}{\mathrm{lin}}
\newcommand{\FM}{\mathrm{flow}}

\newcommand{\euc}{\mathrm{euc}}

\newcommand{\Unif}{\mathrm{Unif}}

\newcommand{\grad}{\nabla_x}
\newcommand{\divx}{\nabla_x\!\cdot}

\newtheorem{theorem}{Theorem}
\newtheorem{proposition}{Proposition}

\newtheorem{corollary}{Corollary}
\theoremstyle{remark}

\newcommand\numberthis{\addtocounter{equation}{1}\tag{\theequation}}
\algrenewcommand\algorithmicindent{12pt}

\title{Co-Evolving Paths and Flows via Path--Flow Alignment}

\author{
  Zeyu Michael Li\thanks{These authors contributed equally to this work.} \\
  Duke University
  \And
  William Xingxu Chen\footnotemark[1] \\
  Duke University
  \And
  Xiang Cheng \\
  Duke University
}

\begin{document}
\maketitle

\begin{abstract}
We study path--flow alignment as a unified training objective for flow matching. Instead of fixing the interpolation path and learning only the velocity field, we jointly train an endpoint-preserving path network and a flow network using the same alignment loss: the flow learns to match the path velocity, and the path learns to align its velocity to the current flow. Although every fixed learned path defines a valid flow-matching objective, the alignment loss alone is not a reliable criterion for path learning. We identify \emph{path overfitting}, a failure mode in which the alignment loss decreases while sample quality worsens. We find that this failure is associated with low-entropy bottlenecks in the induced probability path, where the learned path routes samples through overly concentrated intermediate marginals. Motivated by this diagnosis, we introduce a stochastic path regularizer that hides part of the source information from the path network while preserving exact endpoints. The resulting regularization gives an explicit entropy floor for the stochastic training-path marginals and empirically suppresses the bottleneck in the learned sampler, making joint path--flow training effective. On ImageNet-$256\times256$ with SiT backbones, our method consistently improves FID across model scales, extends to model-guidance training, and leaves the inference-time architecture and sampler unchanged. Code is available at \url{https://github.com/lizeyu090312/traj_opt_paper}.
\end{abstract}

\section{Introduction}
\label{sec:introduction}

Flow matching learns a generative ODE through a supervised regression problem: sample a source point and a data point, interpolate between them, and train a velocity field to match the tangent of that interpolation. This view underlies flow matching and stochastic interpolants \citep{lipman2022flowmatching,albergo2022building,albergo2025stochasticinterpolantsunifyingframework}, and scales to modern image backbones such as SiT \citep{ma2024sit}. In most large-scale image models this interpolation is fixed in advance, often as the linear path
\[
    \gamma_{\lin}(t;x_0,x_1)=(1-t)x_0+tx_1.
\]
The interpolating path is not a passive implementation detail. It determines both the states on which the model is trained and the vector targets it is asked to predict. Prior work has improved flow matching by changing endpoint couplings, straightening trajectories, using metric geometry, or learning curved paths \citep{liu2022rectified,tong2023improving,pooladian2023multisample,kapusniak2024metricflowmatching,shankar2025learningstraightflows}. We ask a complementary question: 

\begin{quote}
\emph{What makes a path good, and can we learn it jointly with the flow?}
\end{quote}

A natural answer is to \emph{use the flow-matching objective itself as a self-supervising signal for both the path and the flow} -- a good path should make the flow-matching regression problem easier. Motivated by this, we propose \emph{path--flow alignment}: we parameterize a path network $\gamma_\psi$, and jointly train the path network and the flow network on the same path--flow alignment loss. We illustrate how learning the interpolation path changes the trajectories of the resulting flow in a toy example (Figure~\ref{fig:toy-learned-trajectories}).
\vspace*{0pt plus 1fill}

\begin{figure}[!t]
\centering
\begin{minipage}[t]{1\linewidth}
    \centering
    \includegraphics[width=0.9\linewidth]{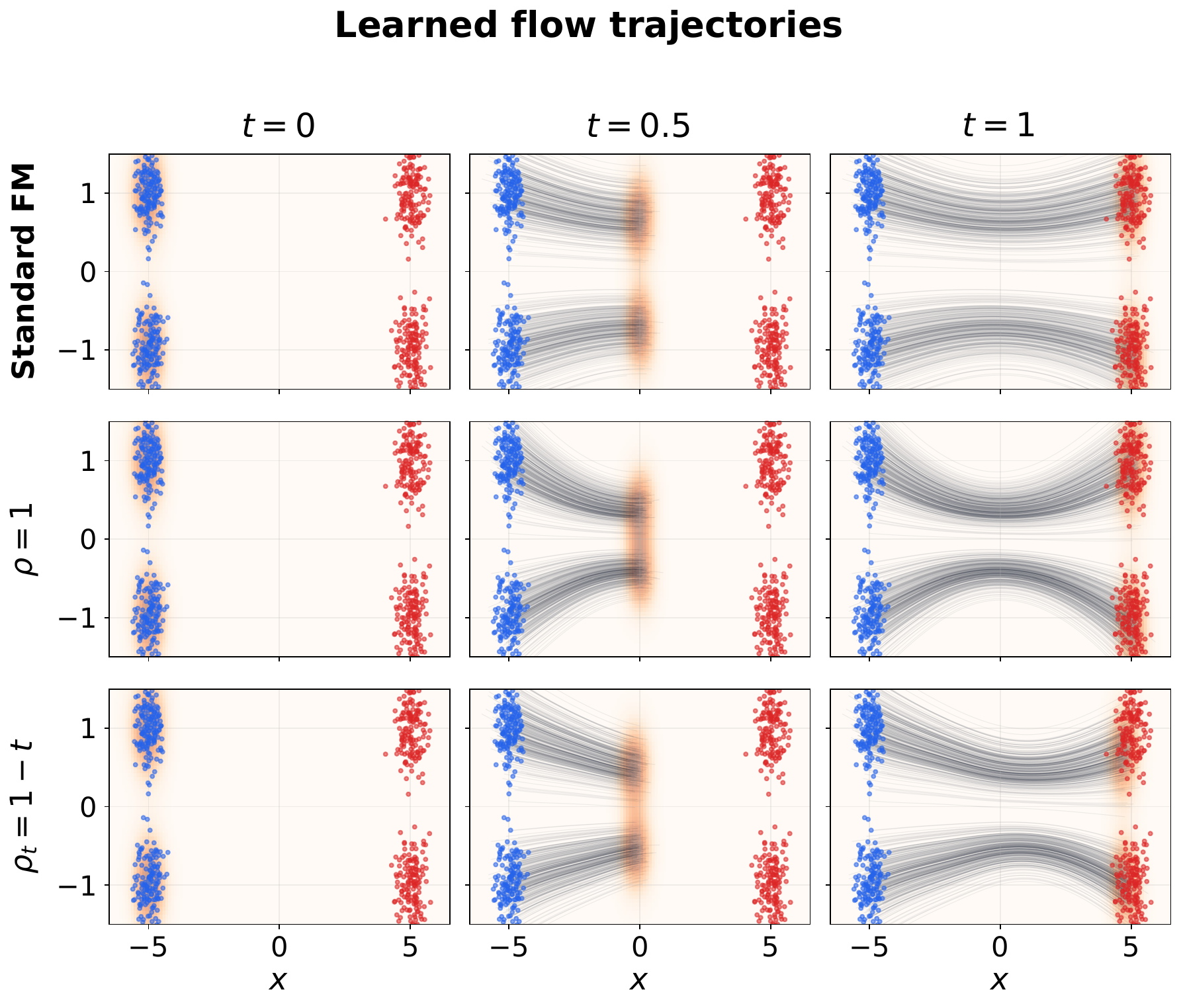}
\end{minipage}
\\
\begin{minipage}[t]{1\linewidth}
    \centering
    \includegraphics[width=0.8\linewidth]{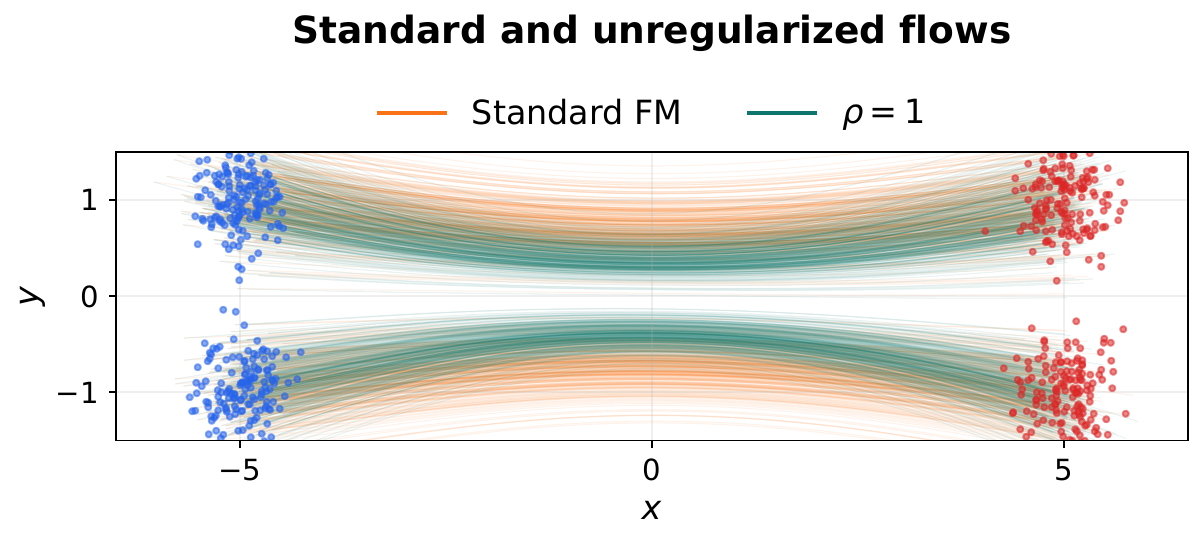}
\end{minipage}
\vspace{-12pt}
\caption{Learned ODE trajectories in a toy mixture example. Top: standard flow matching, path--flow alignment without stochastic regularization ($\rho=1$), and with stochastic regularization ($\rho_t=1-t$), shown at three times. Bottom: standard and unregularized trajectories overlaid. All methods start from the same source samples. Appendix~\ref{app:learned-trajectories} gives further details.}
\label{fig:toy-learned-trajectories}
\end{figure}

For any fixed learned path, we verify that this remains a valid flow-matching problem. Empirically, we observe that this approach quickly improves the flow-matching loss. However, the generation quality soon begins to degrade even as flow-matching loss continues to improve.

We call this \emph{path overfitting}. In our sweeps, such overfitting is consistently accompanied by \emph{entropy collapse} of intermediate marginals -- intuitively, the path network $\gamma_\psi$ is routing most of the flow through a series of low-entropy bottlenecks. This proves our first answer to be incomplete: a good path should not only reduce flow-matching loss, but \emph{do so while avoiding entropy collapse.}

Motivated by this diagnosis, we introduce a simple regularization scheme that applies a Gaussian perturbation to the input of the path network. Theoretically, we show that, under the independent endpoint coupling used here, this regularization establishes an entropy floor on the stochastic training-path marginals. Empirically, we show that this stochastic-regularized path--flow alignment effectively improves the generation quality on the challenging image generation tasks. We also show that the same path--flow alignment template extends beyond standard flow matching by applying it to model-guidance training \citep{tang2025diffusionwithoutcfg}. Overall, our results demonstrate the promise of joint training of path and flow networks, assuming proper regularization.

\subsection{Contributions}
\label{sec:contributions-positioning}

Our main contributions are as follows.

\begin{enumerate}[leftmargin=1.5em,itemsep=2pt,topsep=2pt]
    \item \textbf{Learning interpolation paths for flow matching.}
    We introduce \emph{path--flow alignment} (Algorithm~\ref{alg:simple_alternating_minimization}), a training-time method for learning endpoint-preserving neural paths jointly with a flow model. Section~\ref{sec:flow-guided-interpolation} defines the construction; Proposition~\ref{prop:fm-projection} gives the correctness guarantee.

    \item \textbf{Identifying Path overfitting and low entropy bottlenecks.}
    We identify a failure mode of naive path--flow alignment, which we call \emph{path overfitting}. As the flow-matching loss continues to improve, generation quality initially improves but can later degrade. Sections~\ref{sec:path-overfitting} and \ref{sec:divergence-entropy} empirically associate this failure with a phenomenon that we call a \emph{low-entropy bottleneck}; Proposition~\ref{prop:entropy-rate} gives an exact entropy-rate interpretation of the divergence diagnostic.
    
    \item \textbf{Stochastic Gaussian regularization with entropy guarantees.}
    We design a simple regularization scheme based on perturbing the source input seen by the learned path network. Under independent endpoint coupling, Theorem~\ref{thm:source-randomization-smoothing} lower-bounds the entropy of the stochastic training-path marginals. Proposition~\ref{prop:kinetic-entropy-rate-bound} further lower bounds their instantaneous rate of entropy decrease.
    
    \item \textbf{ImageNet-scale validation and extension to model guidance.}
    On ImageNet-$256\times256$ with SiT backbones, our method consistently improves FID across model sizes while keeping sampling cost unchanged (Table~\ref{tab:model_comparison}). Our alignment template also composes with model-guidance \cite{tang2025diffusionwithoutcfg} training, and shows significant improvement over the model-guidance baseline (Table~\ref{tab:fid_50k_mg}).
\end{enumerate}

\section{Related Work}
\label{sec:related-work}

Flow matching trains continuous-time generative models by regressing a velocity field to tangents of a prescribed probability path \citep{lipman2022flowmatching}. Conditional flow matching, optimal-transport variants, and stochastic interpolants show that the path, coupling, and induced velocity target are central modeling choices \citep{tong2023improving,albergo2022building,albergo2025stochasticinterpolantsunifyingframework}. Our work stays within this framework, and focuses on the concrete question of how to train the path. Prior work improves flow matching by changing endpoint couplings or weights \citep{tong2023improving,pooladian2023multisample,calvoordonez2025weighted,lin2025beyondoptimaltransport}, straightening trajectories through reflow \citep{liu2022rectified}, or replacing Euclidean lines with geometric paths \citep{kapusniak2024metricflowmatching}. Closest to us are methods that learn paths or forward processes, including learned curved paths for straighter flows \citep{shankar2025learningstraightflows}, Neural Flow Diffusion Models \citep{bartosh2024nfdm}, ALI-CFM for observed multi-marginal snapshots \citep{kviman2025multimarginal}, and Curly-FM for non-gradient dynamics from population snapshots \citep{petrovic2025curlyfm}. Our method differs in purpose and use: after endpoints are sampled, we learn an endpoint-conditioned residual path as a training-time supervision mechanism. Rather than optimizing straightness, likelihood, or matching observed intermediate marginals, we optimize for path--flow alignment, and regularize the intermediate marginals via stochastic regularization. Classifier-free guidance modifies sampling by combining conditional and unconditional predictions \citep{cfg_seminal}, while model guidance folds a guided direction into training so that sampling requires a single conditional evaluation \citep{tang2025diffusionwithoutcfg}. Our main experiments use standard CFG evaluation (Section~\ref{sec:experiments}); Section~\ref{sec:mg} shows that the same path--flow alignment template also applies to model-guided velocity objectives.

\section{Flow-Guided Interpolation}
\label{sec:flow-guided-interpolation}

This section defines the learned training path. We first recall the population correctness statement for arbitrary endpoint-preserving paths, then introduce our residual parameterization and the alternating path--flow alignment algorithm.

\subsection{Problem setup and notation}
\label{sec:setup}

Let $p_0=\Ncal(0,I_d)$ be the latent source distribution. For each label $y\in\Ycal$, let $p_1(\cdot\mid y)$ denote the target latent distribution. For simplicity, this paper considers only independently coupled $(x_0,x_1)$, with joint probability $(x_0,x_1\mid y)\sim p_0(x_0) p_1(x_1\mid y)$. The path--flow alignment algorithm and the population flow-matching correctness results below extend to more general endpoint couplings under their stated regularity conditions. However, the Gaussian-smoothing and entropy guarantees in Section~\ref{sec:gaussian-controls-entropy} require an additional conditional-independence condition that is satisfied by the independent coupling above but need not hold for general optimal-transport couplings \citep{tong2023improving,lin2025beyondoptimaltransport}. For any given endpoints $(x_0,x_1)$, define the standard linear interpolation
\begin{equation}
    \gamma_{\lin}(t;x_0,x_1)=(1-t)x_0+tx_1,
    \qquad t\in[0,1].
    \label{eq:xlin}
\end{equation}
The linear path has constant velocity $x_1-x_0$. Standard conditional flow matching trains a velocity network $v_\phi(x,t,y)$ by matching $\dot\gamma_{\lin}(t;x_0,x_1) = x_1 - x_0$ at the intermediate state $\gamma_{\lin}(t;x_0,x_1)$.

More generally, one can replace $\gamma_\lin$ by any sufficiently regular random path \citep{lipman2022flowmatching,albergo2022building,albergo2025stochasticinterpolantsunifyingframework}. Under the regularity conditions of Proposition~\ref{prop:fm-projection}, an endpoint-preserving path induces an ideal ODE transport from source to target. We write a generic path as $\gamma(t;x_0,x_1,y,\xi)$, where $\xi$ captures auxiliary randomness.

\subsection{Flow matching and induced transport}
\label{sec:correctness-induced-flow}
\label{sec:fm-projection-section}
Given any path $\gamma$, let
$p_{\gamma,t}(\cdot\mid y)$ be the distribution of
$\gamma(t;x_0,x_1,y,\xi)$ conditional on $y$, where $\xi$ is sampled from its prescribed auxiliary distribution. The flow-matching loss is
\begin{equation}
    \mathcal L_{\FM}(v_\phi\mid\gamma)
    =\E_{y,x_1,x_0,\xi,t}\Big[
        \norm{v_\phi(\gamma(t;x_0,x_1,y,\xi),t,y)
        -\dot\gamma(t;x_0,x_1,y,\xi)}_2^2
    \Big].
    \label{eq:fm-loss-general}
\end{equation}
If $\gamma(t;x_0,x_1,y,\xi)=\gamma_{\lin}(t;x_0,x_1)$, then $\dot\gamma(t;x_0,x_1,y,\xi)=x_1-x_0$, and \eqref{eq:fm-loss-general} reduces to the standard conditional flow-matching loss $\E_{y,x_1,x_0,t}\Big[
        \norm{v_\phi(\gamma_{\lin}(t;x_0,x_1),t,y)-(x_1-x_0)}_2^2
    \Big]$.

We call $\gamma$ \emph{endpoint-preserving} if, almost surely,
\begin{equation}
    \gamma(0;x_0,x_1,y,\xi)=x_0,
    \qquad
    \gamma(1;x_0,x_1,y,\xi)=x_1.
    \label{eq:general-endpoint-preservation}
\end{equation}

\begin{proposition}[Population target and induced ODE transport]
\label{prop:fm-projection}
\label{prop:correct-population-flow}
Assume
\[
    \E\int_0^1
    \norm{\dot\gamma(t;x_0,x_1,y,\xi)}_2^2\,dt
    <\infty.
\]
For any fixed path $\gamma$, the population minimizer of \eqref{eq:fm-loss-general} is
\begin{equation}
    v_\gamma^\star(x,t,y)
    =\E\!\left[
        \dot\gamma(t;x_0,x_1,y,\xi)
        \;\middle|\;
        \gamma(t;x_0,x_1,y,\xi)=x,\,t,\,y
    \right].
    \label{eq:general-ideal-flow}
\end{equation}

Assume in addition that $t\mapsto\gamma(t;x_0,x_1,y,\xi)$ is almost surely absolutely continuous, that $\gamma$ is endpoint-preserving, and that $x_0\sim p_0$ and $x_1\sim p_1(\cdot\mid y)$. For each fixed $y$, assume that $v_\gamma^\star$ admits a jointly Borel version and, for almost every $t$, satisfies
\[
\begin{aligned}
    \norm{v_\gamma^\star(x,t,y)-v_\gamma^\star(x',t,y)}_2
    &\le L_y(t)\norm{x-x'}_2, \\
    \norm{v_\gamma^\star(x,t,y)}_2
    &\le A_y(t)\bigl(1+\norm{x}_2\bigr),
\end{aligned}
\]
for all $x,x'\in\R^d$, where $L_y,A_y\in L^1([0,1])$. Then the ODE
\begin{equation}
    \frac{d}{dt}z(t)=v_\gamma^\star(z(t),t,y),
    \qquad z(0)\sim p_0,
    \label{eq:ideal-forward-ode}
\end{equation}
has a unique Carath\'eodory flow whose marginals match those of $\gamma(t;x_0,x_1,y,\xi)$ for all $t\in[0,1]$. In particular,
\[
    z(1)\sim p_1(\cdot\mid y).
\]
\end{proposition}
We defer the proof to Appendix~\ref{ss:proof:prop:fm-projection}. Intuitively, \eqref{eq:general-ideal-flow} averages the tangents of all paths passing through the same state, time, and label. Therefore the ODE driven by $v_\gamma^\star$ moves probability mass in the same way as the path marginals.

\subsection{Endpoint-preserving neural paths}
\label{sec:network-interpolation}

Having verified the correctness of the flow-matching objective \eqref{eq:fm-loss-general} for general paths, we now discuss how to use a residual path network $R_\psi(t;x_0,x_1,y):[0,1]\times\R^d\times\R^d\times\Ycal\to\R^d.$ to parameterize a general learnable path:
\begin{equation}
    \gamma_\psi(t;x_0,x_1,y)
    =\gamma_{\lin}(t;x_0,x_1)+t(1-t)R_\psi(t;x_0,x_1,y).
    \label{eq:learned-interpolant}
\end{equation}
We describe the detailed implementation of $R_\psi$ in Appendix~\ref{app:path-subtractor-implementation}; all our results are agnostic to the parameterization of $R_\psi$. The envelope $t(1-t)$ allows the residual to curve the interior of the path while enforcing endpoint preservation \eqref{eq:general-endpoint-preservation}. This envelope has previously been used in \citep{shankar2025learningstraightflows,petrovic2025curlyfm,bartosh2024nfdm}. By observing that $t(1-t) = 0$ at $t=0$ and $t=1$, we immediately verify that 
\begin{equation}
    \gamma_\psi(0;x_0,x_1,y)=x_0,
    \qquad
    \gamma_\psi(1;x_0,x_1,y)=x_1.
    \label{lem:endpoint-preservation}
\end{equation}

\subsection{Flow-guided path alignment and alternating optimization}
\label{sec:teacher-guided-path-alignment}
\label{sec:algorithm-description}

\paragraph{Path and flow are trained on the same objective.} For a fixed path $\gamma_\psi$, the flow phase minimizes
\begin{equation}
    \mathcal L_{\FM}(\phi\mid\psi)
    =\E_{y,x_1,x_0,t}\Big[
        \norm{v_\phi(\gamma_\psi(t;x_0,x_1,y),t,y)
        -\dot\gamma_\psi(t;x_0,x_1,y)}_2^2
    \Big].
    \label{eq:fm-loss-neural}
\end{equation}
The path phase uses the same objective, but freezes $v_\phi$ and optimizes $\gamma_\psi$:
\begin{align*}
    \numberthis \label{eq:path-objective}
    \mathcal L_{\mathrm{path}}(\psi\mid\phi)
    =&\E_{y,x_1,x_0}\int_0^1
        \norm{v_\phi(\gamma_\psi(t;x_0,x_1,y),t,y)-\dot\gamma_\psi(t;x_0,x_1,y)}_2^2\,dt \\
        &\qquad + \beta \E_{y,x_1,x_0}\int_0^1
        \norm{\dot\gamma_\psi(t;x_0,x_1,y)}_2^2\,dt,
    \qquad\qquad \beta\ge0.
    \label{eq:ltrack}
\end{align*}
The first term of \eqref{eq:path-objective} is identical to \eqref{eq:fm-loss-neural}. Thus the path phase does not introduce a separate supervision signal; it uses the current flow model to align the path. The second term of \eqref{eq:path-objective} is the kinetic energy of the path, scaled by $\beta$. This term regularizes the path by driving it toward the straight line; it is also important later for controlling the entropy of the induced path marginals (see Proposition~\ref{prop:kinetic-entropy-rate-bound} of Section~\ref{sec:gaussian-controls-entropy}).

We formally state the procedure in Algorithm~\ref{alg:simple_alternating_minimization} alternates between a \textbf{path training phase} and a \textbf{flow training phase}. In the simple case of $\rho_t\equiv1$, $\hat{x}_0(t) = x_0$, and Algorithm~\ref{alg:simple_alternating_minimization} is ordinary alternating minimization: the path phase aligns $\gamma_\psi$ to the current flow, and the flow phase trains $v_\phi$ on the current path. The optional stochastic-regularization case ($\rho_t<1$) is discussed later in Section~\ref{sec:gaussian-perturbation-definition}.

Figure~\ref{fig:toy-learned-trajectories} illustrates the effect of path--flow alignment on the learned flow in a two-dimensional mixture example. Compared with standard flow matching, the learned flows follow more curved ODE trajectories; stochastic regularization also changes the intermediate marginals. The path network influences these trajectories by changing the states and tangent targets used to train the velocity field, while sampling uses only the learned velocity field. Appendix~\ref{app:learned-trajectories} gives the toy setup and further discussion.

\paragraph{Correctness of flow-matching objective. } Proposition~\ref{prop:fm-projection} and \eqref{lem:endpoint-preservation} together imply that every updated path remains an exact interpolation from \(p_0\) to \(p_1(\cdot\mid y)\), and its population flow-matching target induces the corresponding correct ODE transport. Therefore, after each path update, the subsequent flow network phase is still a valid flow-matching problem for the current path.

\begin{algorithm}[t]
\caption{Alternating path--flow alignment with optional stochastic regularization}
\label{alg:simple_alternating_minimization}
\begin{algorithmic}[1]
\Require Path parameters $\psi$, flow parameters $\phi$, cycles $C$, path steps $K_\psi$, flow steps $K_\phi$, kinetic weight $\beta$, optional perturbation schedule $t\mapsto\rho_t$
\For{$c=1,\dots,C$}
    \For{$k=1,\dots,K_\psi$}
        \State Sample $y$, $x_1\sim p_1(\cdot\mid y)$, $x_0\sim p_0$, $t\sim\Unif[0,1]$, $\epsilon\sim\Ncal(0,I_d)$
        \State $\hat x_0(t)=\rho_t x_0+\sqrt{1-\rho_t^2}\epsilon$
        \State $\hat\gamma_\psi(t;x_0,x_1,y,\epsilon)=(1-t)x_0+tx_1+t(1-t)R_\psi(t;\hat x_0(t),x_1,y).$
        \State $\mathcal J_\psi\gets
        \norm{\partial_t\hat\gamma_\psi(t;x_0,x_1,y,\epsilon)-v_\phi(\hat\gamma_\psi(t;x_0,x_1,y,\epsilon),t,y)}_2^2
        +\beta\norm{\partial_t\hat\gamma_\psi(t;x_0,x_1,y,\epsilon)}_2^2$
        \State Update $\psi$ using $-\nabla_\psi\mathcal J_\psi$, with $\phi$ frozen
    \EndFor
    \For{$k=1,\dots,K_\phi$}
        \State Sample $y$, $x_1\sim p_1(\cdot\mid y)$, $x_0\sim p_0$, $t\sim\Unif[0,1]$, $\epsilon\sim\Ncal(0,I_d)$
        \State $\hat x_0(t)=\rho_t x_0+\sqrt{1-\rho_t^2}\epsilon$
        \State $\hat\gamma_\psi(t;x_0,x_1,y,\epsilon)=(1-t)x_0+tx_1+t(1-t)R_\psi(t;\hat x_0(t),x_1,y).$
        \State $\mathcal J_\phi\gets
        \norm{\partial_t\hat\gamma_\psi(t;x_0,x_1,y,\epsilon)-v_\phi(\hat\gamma_\psi(t;x_0,x_1,y,\epsilon),t,y)}_2^2$
        \State Update $\phi$ using $-\nabla_\phi\mathcal J_\phi$, with $\psi$ frozen
    \EndFor
\EndFor
\end{algorithmic}
\end{algorithm}

\subsection{Path overfitting}
\label{sec:path-overfitting}

\begin{figure}[t]
\centering

\begin{minipage}[t]{0.48\textwidth}
    \centering
    \includegraphics[width=\linewidth]{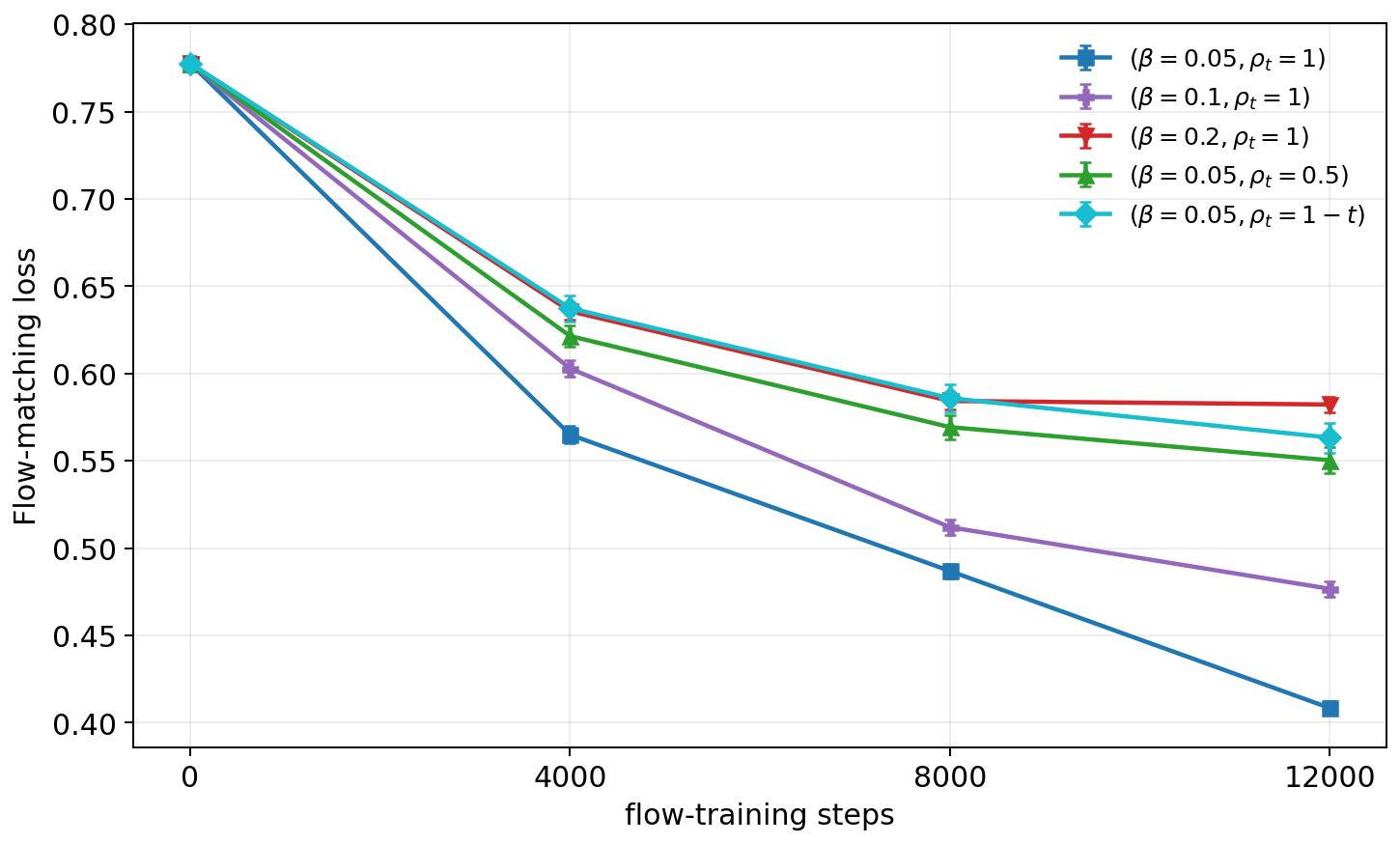}
    \caption{Flow-matching loss across alternating training cycles on the SiT-S model.}
    \label{fig:L_fm_overfitting}
\end{minipage}
\hfill
\begin{minipage}[t]{0.48\textwidth}
    \centering
    \includegraphics[width=\linewidth]{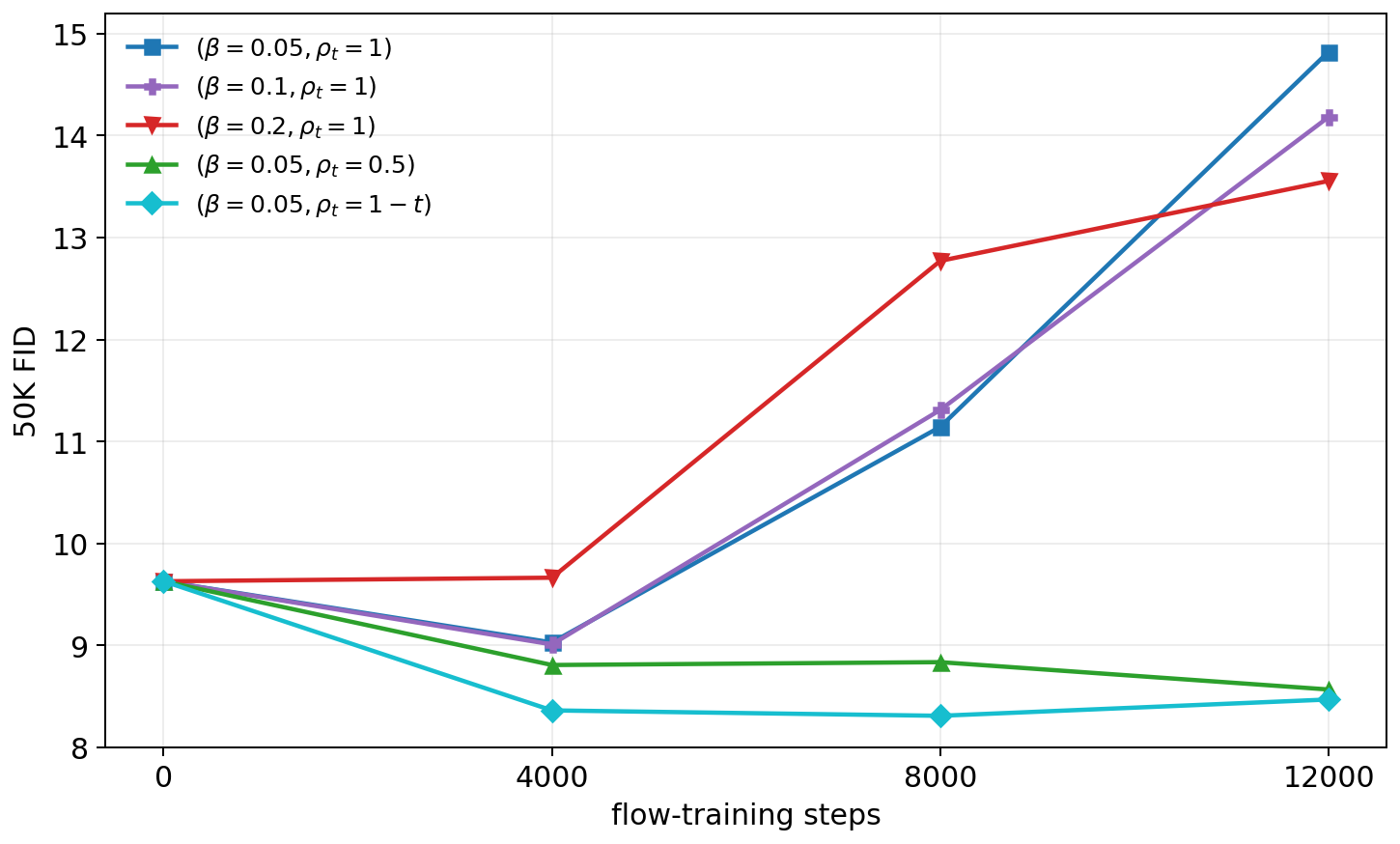}
    \caption{FID across alternating training cycles on the SiT-S model.}
    \label{fig:fid_overfitting}
\end{minipage}

\end{figure}

We first evaluate whether alternating optimization of \(\gamma_\psi\) and \(v_\phi\) improves the downstream flow. Figure~\ref{fig:L_fm_overfitting} shows the flow-matching loss after successive cycles of Algorithm~\ref{alg:simple_alternating_minimization}, and Figure~\ref{fig:fid_overfitting} shows the corresponding FID. We ablate the kinetic energy regularization strength \(\beta\). For reference, the figure also includes the stochastic source-regularized variant introduced later in Section~\ref{sec:overfitting-stochastic}. Appendix~\ref{app:fid-diagnostic-correlations} provides additional scatter-plot diagnostics for the same SiT-S sweep, relating FID to late divergence, average speed, and average acceleration of both the learned flow and the learned path.

The two plots reveal a clear path-overfitting effect. In the $\rho_t \equiv 1$ Algorithm~1 runs, the flow-matching loss continues to decrease as training proceeds. This means that the learned path and the learned flow become increasingly well aligned under the training objective. However, the corresponding FID does not continue to improve. For the unregularized runs, FID initially improves after the first cycle, but then rapidly deteriorates even as \(\mathcal L_{\FM}\) keeps decreasing. 

This demonstrates that a low flow-matching loss is not, by itself, a sufficient condition for a good learned path. Proposition~\ref{prop:fm-projection} guarantees that the population target for an endpoint-preserving path induces a correct transport, but the finite trained network can still be poor for generation in practice. In other words, the path objective can overfit: it can improve the local regression problem without preserving the distributional properties needed for robust sampling.

\section{Entropy Bottlenecks and Stochastic Regularization}
\label{sec:overfitting-stochastic}

Section~\ref{sec:path-overfitting} shows that endpoint correctness and low flow-matching loss do not by themselves guarantee good generation. The learned path can become well aligned with the current flow on its own training states while the resulting sampler develops unstable distributional behavior. In this section, we identify this behavior as an entropy bottleneck and introduce a stochastic regularizer motivated by this diagnosis.

The key diagnostic is average flow divergence. Along any probability flow, average divergence equals the rate of change of differential entropy. Thus a late positive-divergence spike means that the sampler is rapidly expanding entropy near the endpoint. If this expansion compensates for earlier over-contraction, the flow must have passed through an overly concentrated intermediate marginal. We call this a low-entropy bottleneck.

We next introduce stochastic source perturbation: by perturbing the $x_0$ input to the residual path network, we effectively convolve the intermediate path marginals by an independent Gaussian, which in turn lower bounds entropy. The same smoothing also bounds Fisher information; combined with path kinetic energy, this gives an instantaneous lower bound on entropy decrease (equivalently average flow divergence).

To keep the notation separate, Section~\ref{sec:divergence-entropy} diagnoses the probability flow generated by the learned sampler $v_\phi$, whose marginals are denoted $p_{\phi,t}$. Section~\ref{sec:gaussian-perturbation-definition} analyzes the stochastic-regularized training path $\hat\gamma_\psi$, whose marginals are denoted $\hat p_{\psi,t}$. In the population limit where the trained flow equals the path-induced target, $v_\phi=\hat v_\psi^\star$, the sampler marginals $p_{\phi,t}$ coincide with the regularized path marginals $\hat p_{\psi,t}$.

The distinction between these marginals is important for the scope of our results. Theorem~\ref{thm:source-randomization-smoothing} and Proposition~\ref{prop:kinetic-entropy-rate-bound} below control the stochastic training-path marginals $\hat p_{\psi,t}$, whereas Proposition~\ref{prop:entropy-rate} applies directly to the finite learned-sampler marginals $p_{\phi,t}$. The two are guaranteed to coincide when $v_\phi=\hat v_\psi^\star$, but may differ under finite optimization and model capacity. For a fixed path, population excess flow-matching risk controls an $L^2$ velocity discrepancy, while entropy evolution depends on velocity derivatives through divergence; velocity-level control alone therefore does not transfer the path-level entropy guarantee to the learned sampler. We accordingly interpret the divergence measurements below as empirical evidence about the learned sampler rather than as a consequence of the path-level guarantee.

\subsection{Divergence spikes reveal entropy bottlenecks}
\label{sec:divergence-entropy}

An important diagnostic for the path-overfitting phenomenon in Section~\ref{sec:path-overfitting} is the average divergence of the learned velocity field $v_\phi$, which has a direct connection to entropy: for any probability flow, the expected divergence of the velocity field equals the time derivative of differential entropy. Thus a large positive-divergence spike near the endpoint indicates that the learned flow is rapidly increasing entropy at late time. If this late expansion follows an earlier over-contraction, then the probability flow has passed through a low-entropy bottleneck.

\begin{figure}[t]
\centering

\begin{minipage}[t]{0.48\textwidth}
    \centering
    \includegraphics[width=\linewidth]{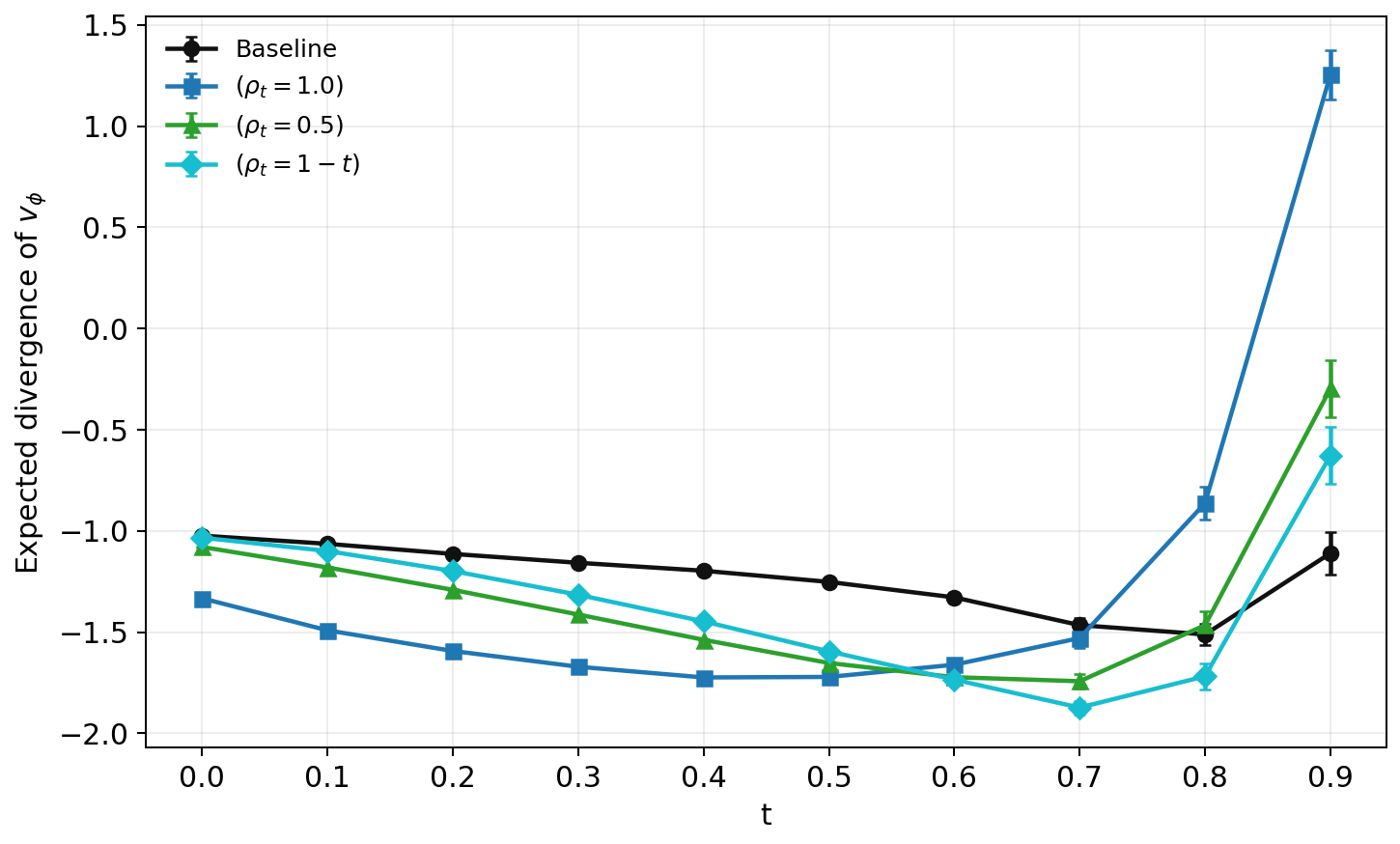}
    \vspace{-20pt}
    \caption{Per-dimensional binned divergence of the learned SiT-S flow network along its induced probability flow.}
    \label{fig:bucket_divergence}
\end{minipage}
\hfill
\begin{minipage}[t]{0.48\textwidth}
    \centering
    \includegraphics[width=\linewidth]{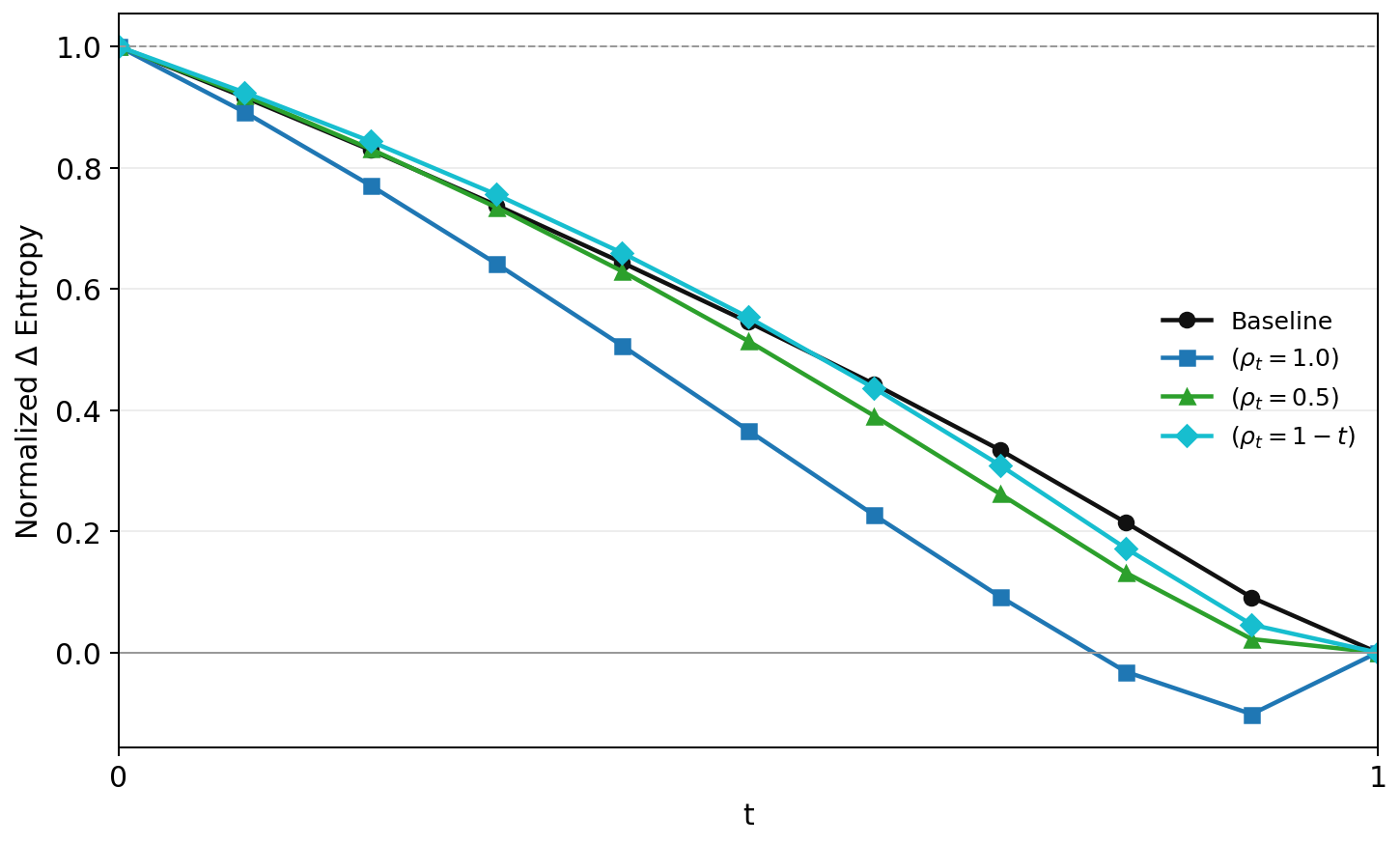}
    \vspace{-20pt}
    \caption{Cumulative normalized divergence. Values below zero indicate an entropy bottleneck relative to the generated endpoint distribution.}
    \label{fig:cumulative_divergence}
\end{minipage}
\end{figure}

For each trained model, let $z_\phi(t)$ denote the ODE trajectory generated by the learned flow,
\begin{equation}
    \frac{d}{dt}z_\phi(t)=v_\phi(z_\phi(t),t,y),
    \qquad z_\phi(0)\sim p_0,
    \label{eq:learned-flow-ode}
\end{equation}
and let $p_{\phi,t}(\cdot\mid y)$ be the law of $z_\phi(t)$ conditional on $y$. We divide the time interval $[0,1]$ into ten bins. For a bin $[a,b]$, we estimate the expected flow divergence 
\begin{equation}
    \bar D_\phi([a,b])
    =
    \frac{1}{d(b-a)}
    \int_a^b
    \E_y\E_{x\sim p_{\phi,t}(\cdot\mid y)}
    \left[
        \divx v_\phi(x,t,y)
    \right]dt,
    \label{eq:binned-divergence-diagnostic}
\end{equation}
using a Hutchinson trace estimator. The factor $1/d$ gives the per-dimensional divergence. Figure~\ref{fig:bucket_divergence} plots this expected flow divergence against the time, for flows $v_\phi$ trained using baseline flow-matching, as well as using Algorithm \ref{alg:simple_alternating_minimization} with different values of $\rho_t$. The overfit $\rho_t=1$ run produces a pronounced positive-divergence spike near $t=1$, in sharp contrast to the baseline and the stochastic source-regularized run introduced in Section~\ref{sec:gaussian-perturbation-definition}. 

The following proposition relates flow-divergence to entropy-change:

\begin{proposition}[Expected divergence equals entropy rate]
\label{prop:entropy-rate}
Let $p_t(\cdot\mid y)$ evolve under the continuity equation $\partial_t p_t(x\mid y)
    +
    \divx\!\left(p_t(x\mid y)u_t(x,y)\right)
    =0,$
with sufficient smoothness, decay at infinity, and finite differential entropy. Equivalently, when the ODE is well posed, $p_t(\cdot\mid y)$ is the law of the solution to
$
    \frac{d}{dt}z(t)=u_t(z(t),y).
$
Then
\begin{equation}
    \frac{d}{dt}h(p_t(\cdot\mid y))
    =
    \E_{x\sim p_t(\cdot\mid y)}
    \left[
        \divx u_t(x,y)
    \right].
    \label{eq:generic-entropy-rate}
\end{equation}
Consequently, for any time bin $0\le a<b\le1$,
\begin{equation}
    \frac{1}{b-a}
    \int_a^b
    \E_{x\sim p_t(\cdot\mid y)}
    \left[
        \divx u_t(x,y)
    \right]dt
    =
    \frac{h(p_b(\cdot\mid y))-h(p_a(\cdot\mid y))}{b-a}.
    \label{eq:generic-binned-entropy-divergence}
\end{equation}
\end{proposition}
We defer the proof to Appendix \ref{ss:proof:prop:entropy-rate}. We apply Proposition~\ref{prop:entropy-rate} directly to the learned flow by taking
$
    u_t(\cdot,y)=v_\phi(\cdot,t,y),
    p_t(\cdot\mid y)=p_{\phi,t}(\cdot\mid y).
$
This gives the exact identity
\begin{equation}
    \frac{d}{dt}h(p_{\phi,t}(\cdot\mid y))
    =
    \E_{x\sim p_{\phi,t}(\cdot\mid y)}
    \left[
        \divx v_\phi(x,t,y)
    \right].
    \label{eq:learned-flow-entropy-rate}
\end{equation}
Thus the divergence curves in Figure~\ref{fig:bucket_divergence} are not merely local geometric measurements of $v_\phi$; they estimate the entropy-rate profile of the probability flow generated by the learned sampler. If, in the population limit, $v_\phi$ equals the ideal path-induced velocity $v_\psi^\star$ from Proposition~\ref{prop:fm-projection}, then $p_{\phi,t}=p_{\psi,t}$ by the induced ODE transport statement, and the generated endpoint is $p_1(\cdot\mid y)$.

The identity also gives a precise meaning to a low-entropy bottleneck. For any $s<1$, applying \eqref{eq:learned-flow-entropy-rate} on $[s,1]$ gives
$
    h(p_{\phi,1}(\cdot\mid y))-h(p_{\phi,s}(\cdot\mid y))
    =
    \int_s^1
    \E_{x\sim p_{\phi,t}(\cdot\mid y)}
    \left[
        \divx v_\phi(x,t,y)
    \right]dt.
$
Therefore,
\begin{equation}
    \int_s^1
    \E_{x\sim p_{\phi,t}(\cdot\mid y)}
    \left[
        \divx v_\phi(x,t,y)
    \right]dt
    >0
    \quad\Longleftrightarrow\quad
    h(p_{\phi,s}(\cdot\mid y))<h(p_{\phi,1}(\cdot\mid y)).
    \label{eq:entropy-bottleneck-equivalence}
\end{equation}
We call this event a \emph{low-entropy bottleneck}: the learned probability flow has contracted below the entropy of its own endpoint distribution $p_{\phi,1}$ and must subsequently expand entropy before $t=1$. In the ideal endpoint-preserving case $p_{\phi,1}=p_1$, so the same condition says that the intermediate marginal has entropy below the target entropy.

Figure~\ref{fig:cumulative_divergence} visualizes this effect through cumulative divergence. Let $\bar D_\phi(t)$ denote the piecewise-constant, per-dimensional binned estimate from \eqref{eq:binned-divergence-diagnostic}. We plot
$
    \widehat G_\phi(t)
    =
    1-
    (\int_0^t \bar D_\phi(r)\,dr)/(\int_0^1 \bar D_\phi(r)\,dr).
$
For the exact entropy-rate curve, this quantity equals the normalized excess entropy above the endpoint,
\begin{equation}
    G_\phi(t;y)
    =
    \frac{h(p_{\phi,t}(\cdot\mid y))-h(p_{\phi,1}(\cdot\mid y))}
         {h(p_0)-h(p_{\phi,1}(\cdot\mid y))}
    =
    1-
    \frac{\int_0^t
    \E_{x\sim p_{\phi,r}(\cdot\mid y)}[\divx v_\phi(x,r,y)]\,dr}
    {\int_0^1
    \E_{x\sim p_{\phi,r}(\cdot\mid y)}[\divx v_\phi(x,r,y)]\,dr}.
    \label{eq:normalized-excess-entropy}
\end{equation}
The plotted statistic estimates the divergences with Hutchinson traces. A value near $1$ corresponds to source-side entropy, a value near $0$ corresponds to reaching the endpoint entropy, and a value below $0$ indicates entropy below the endpoint entropy. The overfit run dips below zero around $t\approx0.9$, matching the large positive-divergence spike in Figure~\ref{fig:bucket_divergence}. This supports the interpretation that the overfit flow passes through a low-entropy bottleneck and then requires a compensating expansion.

\subsection{Regularization via Gaussian perturbation}
\label{sec:gaussian-perturbation-definition}
\label{sec:gaussian-controls-entropy}

Section~\ref{sec:divergence-entropy} suggests that path overfitting is associated with entropy bottlenecks: the learned path can reduce its local regression objective while routing probability mass through overly concentrated intermediate marginals. We therefore regularize the residual path network by hiding part of the source noise from the residual input. The key design choice is that the perturbation is applied only to the input of $R_\psi$; the endpoint interpolation itself still uses the true endpoints $(x_0,x_1)$.

Let $\rho_t\in[0,1]$ be a deterministic perturbation schedule and let $\epsilon\sim\Ncal(0,I_d)$ be independent of $(x_0,x_1,y)$. Define
\begin{equation}
    \hat x_0(t)
    =\rho_t x_0+\sqrt{1-\rho_t^2}\,\epsilon .
    \label{eq:time-dependent-xhat0}
\end{equation}
For each fixed $t$, the perturbed source input is still marginally standard Gaussian,
$
    \hat x_0(t)\sim\Ncal(0,I_d),
$
but its correlation with the true source endpoint is reduced from $1$ to $\rho_t$. Thus the residual network receives a source-like input with less information about the particular endpoint noise $x_0$ that appears in the actual interpolation. The stochastic-regularized path is
\begin{equation}
    \hat\gamma_{\psi}(t;x_0,x_1,y,\epsilon)
    =(1-t)x_0+tx_1+t(1-t)R_\psi(t;\hat x_0(t),x_1,y).
    \label{eq:time-dependent-rho-interpolant}
\end{equation}
This is the stochastic path used in Algorithm~\ref{alg:simple_alternating_minimization}. Setting $\rho_t\equiv1$ gives $\hat x_0(t)=x_0$ and recovers the unregularized path $\gamma_\psi$. Setting $\rho_t<1$ hides part of $x_0$ from the residual network, while leaving the true linear endpoint term unchanged. When $\rho_t$ is time-dependent, the tangent $\partial_t\hat\gamma_\psi$ differentiates through both the envelope $t(1-t)$ and the schedule $\rho_t$ in \eqref{eq:time-dependent-xhat0}. The perturbation does not change the endpoints. For every schedule $\rho_t$, every $\epsilon$, and every residual network $R_\psi$,
\begin{equation}
    \hat\gamma_\psi(0;x_0,x_1,y,\epsilon)=x_0,
    \qquad
    \hat\gamma_\psi(1;x_0,x_1,y,\epsilon)=x_1.
    \label{eq:rho-endpoint-preservation}
\end{equation}
Only the residual term sees the corrupted source input, and this residual term is multiplied by $t(1-t)$, which vanishes at both endpoints. When path learning is disabled, $R_\psi\equiv 0$, so \eqref{eq:time-dependent-rho-interpolant} reduces to the standard linear path independently of $\rho_t$; stochastic source perturbation therefore has no standalone effect without a learned residual path.

Let $\hat p_{\psi,t}(\cdot\mid y)$ denote the distribution of
$\hat\gamma_\psi(t;x_0,x_1,y,\epsilon)$ conditional on $y$, where $(x_0,x_1)$ is sampled from the chosen endpoint coupling and $\epsilon\sim\Ncal(0,I_d)$. The corresponding population flow-matching target is
\begin{equation}
    \hat v_{\psi}^\star(x,t,y)
    =\E\left[
        \partial_t\hat\gamma_{\psi}(t;x_0,x_1,y,\epsilon)
        \;\middle|\;
        \hat\gamma_{\psi}(t;x_0,x_1,y,\epsilon)=x,\,t,\,y
    \right].
    \label{eq:rho-ideal-flow}
\end{equation}
Thus Gaussian perturbation changes the training path distribution, but it does not make the flow-matching target biased or approximate: the population target remains the exact marginal velocity of the perturbed path.

\begin{corollary}[Correctness of stochastic source perturbation]
\label{cor:stochastic-source-correctness}
Assume that the stochastic path \eqref{eq:time-dependent-rho-interpolant} satisfies the assumptions of Proposition~\ref{prop:fm-projection}. Then the flow-matching loss trained on pairs $
    \left(
    \hat\gamma_\psi(t;x_0,x_1,y,\epsilon),
    \partial_t\hat\gamma_\psi(t;x_0,x_1,y,\epsilon)
    \right)
$
has population minimizer \eqref{eq:rho-ideal-flow}. Moreover, the probability flow generated by $\hat v_\psi^\star$ transports $p_0$ to $p_1(\cdot\mid y)$.
\end{corollary}

\begin{proof}
The additional Gaussian variable $\epsilon$ is simply auxiliary randomness in the sense of Section~\ref{sec:correctness-induced-flow}. Endpoint preservation follows from \eqref{eq:rho-endpoint-preservation}. The Corollary follows immediately from Proposition~\ref{prop:fm-projection}.
\end{proof}

We now explain why the stochastic source perturbation in Section~\ref{sec:gaussian-perturbation-definition} directly regularizes the intermediate marginals: First, at any fixed time $t$ away from the endpoint, the part of $x_0$ hidden from the residual network can be rewritten as independent additive Gaussian noise in the state. Thus $\hat p_{\psi,t}(\cdot\mid y)$ is a Gaussian-smoothed distribution and has an explicit entropy floor. Second, the same smoothing gives a Fisher-information bound. Combining this Fisher-information bound with the pointwise kinetic energy of the path yields an instantaneous lower bound on the entropy rate, or equivalently on the average divergence of the ideal marginal velocity.

The correctness results above do not require the independent endpoint coupling. For the Gaussian-smoothing result below, however, we use the independent coupling from Section~\ref{sec:setup}: conditional on $y$, $x_0\sim\Ncal(0,I_d)$ and $x_1\sim p_1(\cdot\mid y)$ are independent, and $\epsilon\sim\Ncal(0,I_d)$ is independent of both. More generally, a sufficient condition for the argument below is that $\eta_t$ in \eqref{eq:eta-t-definition} is conditionally standard Gaussian and independent of $(\hat x_0(t),x_1)$ given $y$. A general optimal-transport coupling need not satisfy this condition.

\begin{theorem}[Source randomization induces Gaussian smoothing]
\label{thm:source-randomization-smoothing}
Fix a time $t\in[0,1)$ such that $\rho_t<1$, and define
\begin{equation}
    \sigma_t=(1-t)\sqrt{1-\rho_t^2}.
    \label{eq:sigma-t-definition}
\end{equation}
Let $\hat p_{\psi,t}(\cdot\mid y)$ be the law of the stochastic path
\eqref{eq:time-dependent-rho-interpolant} conditional on $y$. Then there exists a distribution $q_{\psi,t}(\cdot\mid y)$ such that
\begin{equation}
    \hat p_{\psi,t}(\cdot\mid y)
    =q_{\psi,t}(\cdot\mid y)*\Ncal(0,\sigma_t^2 I_d).
    \label{eq:rho-convolution}
\end{equation}
More explicitly, if
\begin{equation}
    \eta_t
    =\frac{x_0-\rho_t\hat x_0(t)}{\sqrt{1-\rho_t^2}},
    \label{eq:eta-t-definition}
\end{equation}
then, conditional on $y$, $\eta_t\sim\Ncal(0,I_d)$ and $\eta_t$ is independent of $(\hat x_0(t),x_1)$, and
\begin{equation}
\begin{aligned}
    \hat\gamma_\psi(t;x_0,x_1,y,\epsilon)
    &=(1-t)\rho_t\hat x_0(t)+tx_1+t(1-t)R_\psi(t;\hat x_0(t),x_1,y) +\sigma_t\eta_t.
\end{aligned}
\label{eq:rho-gamma-convolution-form}
\end{equation}
The distribution $q_{\psi,t}(\cdot\mid y)$ in \eqref{eq:rho-convolution} is the law, conditional on $y$, of the non-Gaussian term
\begin{equation}
    (1-t)\rho_t\hat x_0(t)+tx_1+t(1-t)R_\psi(t;\hat x_0(t),x_1,y).
    \label{eq:qrho-pushforward-variable}
\end{equation}
Consequently, whenever the differential entropy is well-defined,
\begin{equation}
    h(\hat p_{\psi,t}(\cdot\mid y))
    \ge
    \frac d2\log\!\left(2\pi e\,\sigma_t^2\right)
    =
    \frac d2\log\!\left(2\pi e(1-t)^2(1-\rho_t^2)\right).
    \label{eq:entropy-floor}
\end{equation}
Moreover, the Fisher information of the smoothed marginal satisfies
\begin{equation}
    I(\hat p_{\psi,t}(\cdot\mid y))
    :=
    \E_{x\sim\hat p_{\psi,t}(\cdot\mid y)}
    \left[\norm{\grad\log \hat p_{\psi,t}(x\mid y)}_2^2\right]
    \le
    \frac{d}{\sigma_t^2}
    =
    \frac{d}{(1-t)^2(1-\rho_t^2)}.
    \label{eq:fisher-gaussian-smoothing-bound}
\end{equation}
\end{theorem}

We defer the proof to Appendix \ref{sec:proof:thm:source-randomization-smoothing}. The entropy floor in \eqref{eq:entropy-floor} is a state-level obstruction to arbitrarily deep entropy collapse. The bound becomes vacuous as $t\to1$, which is unavoidable: endpoint preservation requires the added smoothing to vanish at the target endpoint. %

This smoothing result also gives the bin-level intuition behind the divergence diagnostics above. Applying Proposition~\ref{prop:entropy-rate} to the ideal velocity $\hat v_\psi^\star$ gives, for any $0\le a<b<1$,
\[
    \int_a^b
    \E_{x\sim\hat p_{\psi,t}(\cdot\mid y)}
    [\divx \hat v_\psi^\star(x,t,y)]\,dt
    =
    h(\hat p_{\psi,b}(\cdot\mid y))-h(\hat p_{\psi,a}(\cdot\mid y)).
\]
Thus binned average divergence is a finite-difference entropy rate. Since \eqref{eq:entropy-floor} lower-bounds the entropy at the end of any bin with $\sigma_b>0$, Gaussian perturbation prevents arbitrarily deep entropy collapse on such a bin unless the entropy at the beginning of the bin is correspondingly large.

We next give an instantaneous version of this control:

\begin{proposition}[Entropy-rate lower bound under kinetic-energy control]
\label{prop:kinetic-entropy-rate-bound}
Assume the regularity conditions needed for Proposition~\ref{prop:entropy-rate} and the coupling and regularity conditions of Theorem~\ref{thm:source-randomization-smoothing}. Define the pointwise stochastic-path kinetic energy
\begin{equation}
    L_{\euc}^{\rho,\psi}(t,y)
    :=
    \E\left[
        \norm{\partial_t\hat\gamma_\psi(t;x_0,x_1,y,\epsilon)}_2^2
        \;\middle|\; y
    \right],
    \label{eq:pointwise-leuc-rho}
\end{equation}
where the superscript $\rho$ denotes the schedule $t\mapsto\rho_t$. Then, for every $t<1$ with $\rho_t<1$,
\begin{equation}
    \frac{d}{dt}h(\hat p_{\psi,t}(\cdot\mid y))
    \ge
    -\frac{\sqrt{d\,L_{\euc}^{\rho,\psi}(t,y)}}{(1-t)\sqrt{1-\rho_t^2}}
    =
    -\frac{d}{1-t}
    \sqrt{\frac{L_{\euc}^{\rho,\psi}(t,y)}{d(1-\rho_t^2)}}.
    \label{eq:entropy-rate-kinetic-bound}
\end{equation}
\end{proposition}

We defer the proof to Appendix \ref{sec:proof:prop:kinetic-entropy-rate-bound}. Proposition~\ref{prop:kinetic-entropy-rate-bound} shows that stochastic perturbation and kinetic regularization play complementary roles. The smoothing scale
$
    \sigma_t=(1-t)\sqrt{1-\rho_t^2}
$
controls the Fisher information of the marginal: larger smoothing makes the score smaller and prevents sharper entropy collapse. The kinetic term controls the size of the marginal velocity acting on this smoothed distribution. Thus $L_{\euc}$ is not only a geometric straightness penalty; in the presence of Gaussian perturbation, it also bounds how fast the path-induced marginal can decrease entropy.

\paragraph{The scaling is comparable to the standard linear path.} If $R_\psi\equiv0$ and the target is a point mass, then the linear path has marginal $\Ncal(0,(1-t)^2I_d)$ and entropy rate exactly $-d/(1-t)$. The bound \eqref{eq:entropy-rate-kinetic-bound} recovers this order when $\rho_t$ is bounded away from $1$. Conversely, when $\rho_t$ is close to $1$, there is little smoothing and the bound appropriately becomes weak.

\section{Generalizing Beyond Flow Matching}
\label{sec:mg}

Path--flow alignment is not tied to the standard flow-matching velocity; it only needs a velocity-like target that is regressed to the path tangent. Model guidance \citep{tang2025diffusionwithoutcfg} provides such a target by folding classifier-free guidance into training. For guidance scale $\omega$, define
\[
    \tilde v_\phi^\omega(x,t,y)
    =v_\phi(x,t,y)
    -\omega\,\mathrm{sg}\!\left(v_\phi(x,t,y)-v_\phi(x,t,\emptyset)\right),
\]
where $\mathrm{sg}$ denotes stop-gradient, and train with
\begin{equation}
    \mathcal L_{\mathrm{MG}}(\phi\mid\gamma)
    =\E_{y,x_1,x_0,\xi,t}
    \norm{\tilde v_\phi^\omega(\gamma(t;x_0,x_1,y,\xi),t,y)
    -\dot\gamma(t;x_0,x_1,y,\xi)}_2^2 .
    \label{eq:model_guidance_loss}
\end{equation}
We use the same alternating updates as Algorithm~\ref{alg:simple_alternating_minimization}, replacing $v_\phi$ by $\tilde v_\phi^\omega$ in both phases. Stochastic source perturbation and endpoint preservation are unchanged. Table~\ref{tab:fid_50k_mg} shows that this extension improves 50k-FID over standard MG fine-tuning on both SiT-B-MG and SiT-XL-MG.

\section{Experiments}
\label{sec:experiments}
Unless otherwise stated, we use $\beta=0.05$ and $\rho_t = c(1-t)$ with $c=1$ for SiT-S/B/L and $c=0.7$ for SiT-XL. We provide full training setups in Appendix~\ref{sec:experimental_details_appx}. 
\paragraph{ImageNet results.}
We evaluate regularized path--flow alignment on class-conditional ImageNet-$256\times256$ using SiT backbones \citep{ma2024sit,imagenet_dataset}. Each run starts from the same pretrained SiT checkpoint. The baseline continues standard flow-matching training for 12k steps; our method uses three alternating cycles, each with 3k path updates followed by 4k flow updates. Table~\ref{tab:model_comparison} reports the best 50k-FID achieved within the 12k flow-update budget. Our method improves FID at every scale, including SiT-XL, where standard fine-tuning does not improve the starting checkpoint. The largest gains are on SiT-L and SiT-XL, with relative FID improvements of $19.1\%$ and $12.7\%$, respectively. The full checkpoint-wise FID/IS table is in Appendix Table~\ref{tab:model_comparison_full}.

\paragraph{Model guidance.}
We also test the effective-velocity extension from Section~\ref{sec:mg}. Starting from FID-converged SiT-B-MG and SiT-XL-MG checkpoints, we compare standard MG fine-tuning baseline against path--flow alignment with the MG objective. Table~\ref{tab:fid_50k_mg} shows that our method improves the best FID from $7.940$ to $5.585$ on SiT-B-MG and from $1.494$ to $1.322$ on SiT-XL-MG. Checkpoint-wise FID results are in Appendix Table~\ref{tab:fid_50k_mg_full}.

The FID improvements persist against compute-matched continued finetuning (Appendix~\ref{app:compute-matched-continuation}) and across three disjoint generation-seed sets (Appendix~\ref{app:generation-seed-robustness}). The ImageNet results improve recall and Inception Score (Appendix~\ref{app:cfg-and-pr} and Table~\ref{tab:model_comparison_full}). We also find gains on FlowDCN-B (Appendix~\ref{app:flowdcn-transfer}) and compare with \citep{shankar2025learningstraightflows} (Appendix~\ref{app:cifar-controlled-comparison}). We demonstrate that path--flow alignment increases curvature of sampled trajectories on a toy setup and include qualitative comparisons showing improvements in image quality (Appendix~\ref{app:qualitative-samples}).

\begin{table}[H]
\centering
\small
\caption{ImageNet-$256\times256$ summary. We report the best 50k-FID ($\downarrow$) over the same 12k continuation budget; full checkpoint-wise FID/IS results are in Appendix Tables~\ref{tab:model_comparison_full}--\ref{tab:fid_50k_mg_full}. Params denote the inference-time SiT flow model only; the path network is used only during training.}
\label{tab:model_comparison}
\label{tab:fid_50k_mg}
\begin{tabular}{llccccc}
\toprule
Objective & Model & Params & Init FID & Baseline best & Ours best & FID gain \\
\midrule
FM & SiT-S  & 33M  & 9.109 & 9.022 & \textbf{8.359} & $7.3\%$ \\
FM & SiT-B  & 130M & 5.074 & 5.057 & \textbf{4.415} & $12.7\%$ \\
FM & SiT-L  & 458M & 3.424 & 3.390 & \textbf{2.742} & $19.1\%$ \\
FM & SiT-XL & 675M & 2.091 & 2.091 & \textbf{1.825} & $12.7\%$ \\
MG & SiT-B-MG  & 130M & 8.090 & 7.940 & \textbf{5.585} & $29.7\%$ \\
MG & SiT-XL-MG & 675M & 1.522 & 1.494 & \textbf{1.322} & $11.5\%$ \\
\bottomrule
\end{tabular}
\end{table}

\section{Conclusion}
We introduced path--flow alignment, which jointly learns an endpoint-preserving training path and flow model, and identified path overfitting as a failure mode of naive joint optimization. Stochastic source perturbation provides Gaussian-smoothing and entropy guarantees for the learned training path, while our experiments show that it also suppresses the associated late-time divergence bottleneck in finite learned samplers. Across SiT and FlowDCN models, the resulting post-training procedure improves generation quality under compute-matched comparisons without changing the inference architecture or sampler. Our theory does not establish a causal link between entropy bottlenecks and generation quality, nor does it guarantee transfer of the path-level entropy control to an imperfectly fitted flow; understanding this connection and extending joint path--flow training beyond pretrained image generators remain important directions.

\section*{Funding and Competing Interests}
The authors received no third-party funding or support for this work, and declare no competing financial interests.

\bibliographystyle{plain}
\bibliography{updated_references}

@misc{lipman2022flowmatching,
      title={Flow Matching for Generative Modeling}, 
      author={Yaron Lipman and Ricky T. Q. Chen and Heli Ben-Hamu and Maximilian Nickel and Matt Le},
      year={2023},
      eprint={2210.02747},
      archivePrefix={arXiv},
      primaryClass={cs.LG},
      url={https://arxiv.org/abs/2210.02747}, 
}

@misc{ma2024sit,
      title={SiT: Exploring Flow and Diffusion-based Generative Models with Scalable Interpolant Transformers}, 
      author={Nanye Ma and Mark Goldstein and Michael S. Albergo and Nicholas M. Boffi and Eric Vanden-Eijnden and Saining Xie},
      year={2024},
      eprint={2401.08740},
      archivePrefix={arXiv},
      primaryClass={cs.CV},
      url={https://arxiv.org/abs/2401.08740}, 
}

@misc{bartosh2024nfdm,
      title={Neural Flow Diffusion Models: Learnable Forward Process for Improved Diffusion Modelling}, 
      author={Grigory Bartosh and Dmitry Vetrov and Christian A. Naesseth},
      year={2025},
      eprint={2404.12940},
      archivePrefix={arXiv},
      primaryClass={stat.ML},
      url={https://arxiv.org/abs/2404.12940}, 
}

@misc{kapusniak2024metricflowmatching,
      title={Metric Flow Matching for Smooth Interpolations on the Data Manifold}, 
      author={Kacper Kapuśniak and Peter Potaptchik and Teodora Reu and Leo Zhang and Alexander Tong and Michael Bronstein and Avishek Joey Bose and Francesco Di Giovanni},
      year={2024},
      eprint={2405.14780},
      archivePrefix={arXiv},
      primaryClass={cs.LG},
      url={https://arxiv.org/abs/2405.14780}, 
}

@misc{lin2025beyondoptimaltransport,
      title={Beyond Optimal Transport: Model-Aligned Coupling for Flow Matching}, 
      author={Yexiong Lin and Yu Yao and Tongliang Liu},
      year={2025},
      eprint={2505.23346},
      archivePrefix={arXiv},
      primaryClass={cs.CV},
      url={https://arxiv.org/abs/2505.23346}, 
}

@misc{shankar2025learningstraightflows,
      title={Learning Straight Flows by Learning Curved Interpolants}, 
      author={Shiv Shankar and Tomas Geffner},
      year={2025},
      eprint={2503.20719},
      archivePrefix={arXiv},
      primaryClass={cs.LG},
      url={https://arxiv.org/abs/2503.20719}, 
}

@misc{kviman2025multimarginal,
      title={Multi-Marginal Flow Matching with Adversarially Learnt Interpolants}, 
      author={Oskar Kviman and Kirill Tamogashev and Nicola Branchini and Víctor Elvira and Jens Lagergren and Nikolay Malkin},
      year={2026},
      eprint={2510.01159},
      archivePrefix={arXiv},
      primaryClass={cs.LG},
      url={https://arxiv.org/abs/2510.01159}, 
}

@misc{petrovic2025curlyfm,
      title={Curly Flow Matching for Learning Non-gradient Field Dynamics}, 
      author={Katarina Petrović and Lazar Atanackovic and Viggo Moro and Kacper Kapuśniak and İsmail İlkan Ceylan and Michael Bronstein and Avishek Joey Bose and Alexander Tong},
      year={2025},
      eprint={2510.26645},
      archivePrefix={arXiv},
      primaryClass={cs.LG},
      url={https://arxiv.org/abs/2510.26645}, 
}

@misc{tong2023improving,
      title={Improving and generalizing flow-based generative models with minibatch optimal transport}, 
      author={Alexander Tong and Kilian Fatras and Nikolay Malkin and Guillaume Huguet and Yanlei Zhang and Jarrid Rector-Brooks and Guy Wolf and Yoshua Bengio},
      year={2024},
      eprint={2302.00482},
      archivePrefix={arXiv},
      primaryClass={cs.LG},
      url={https://arxiv.org/abs/2302.00482}, 
}

@misc{albergo2022building,
      title={Building Normalizing Flows with Stochastic Interpolants}, 
      author={Michael S. Albergo and Eric Vanden-Eijnden},
      year={2023},
      eprint={2209.15571},
      archivePrefix={arXiv},
      primaryClass={cs.LG},
      url={https://arxiv.org/abs/2209.15571}, 
}

@misc{albergo2025stochasticinterpolantsunifyingframework,
      title={Stochastic Interpolants: A Unifying Framework for Flows and Diffusions}, 
      author={Michael S. Albergo and Nicholas M. Boffi and Eric Vanden-Eijnden},
      year={2025},
      eprint={2303.08797},
      archivePrefix={arXiv},
      primaryClass={cs.LG},
      url={https://arxiv.org/abs/2303.08797}, 
}

@misc{liu2022rectified,
      title={Flow Straight and Fast: Learning to Generate and Transfer Data with Rectified Flow}, 
      author={Xingchao Liu and Chengyue Gong and Qiang Liu},
      year={2022},
      eprint={2209.03003},
      archivePrefix={arXiv},
      primaryClass={cs.LG},
      url={https://arxiv.org/abs/2209.03003}, 
}

@misc{pooladian2023multisample,
      title={Multisample Flow Matching: Straightening Flows with Minibatch Couplings}, 
      author={Aram-Alexandre Pooladian and Heli Ben-Hamu and Carles Domingo-Enrich and Brandon Amos and Yaron Lipman and Ricky T. Q. Chen},
      year={2023},
      eprint={2304.14772},
      archivePrefix={arXiv},
      primaryClass={cs.LG},
      url={https://arxiv.org/abs/2304.14772}, 
}

@misc{calvoordonez2025weighted,
      title={Weighted Conditional Flow Matching}, 
      author={Sergio Calvo-Ordonez and Matthieu Meunier and Alvaro Cartea and Christoph Reisinger and Yarin Gal and Jose Miguel Hernandez-Lobato},
      year={2026},
      eprint={2507.22270},
      archivePrefix={arXiv},
      primaryClass={cs.LG},
      url={https://arxiv.org/abs/2507.22270}, 
}

@misc{tang2025diffusionwithoutcfg,
      title={Diffusion Models without Classifier-free Guidance}, 
      author={Zhicong Tang and Jianmin Bao and Dong Chen and Baining Guo},
      year={2025},
      eprint={2502.12154},
      archivePrefix={arXiv},
      primaryClass={cs.CV},
      url={https://arxiv.org/abs/2502.12154}, 
}

@INPROCEEDINGS{imagenet_dataset,
  author={Deng, Jia and Dong, Wei and Socher, Richard and Li, Li-Jia and Kai Li and Li Fei-Fei},
  booktitle={2009 IEEE Conference on Computer Vision and Pattern Recognition}, 
  title={ImageNet: A large-scale hierarchical image database}, 
  year={2009},
  volume={},
  number={},
  pages={248-255},
  doi={10.1109/CVPR.2009.5206848}}

@misc{cfg_seminal,
      title={Classifier-Free Diffusion Guidance}, 
      author={Jonathan Ho and Tim Salimans},
      year={2022},
      eprint={2207.12598},
      archivePrefix={arXiv},
      primaryClass={cs.LG},
      url={https://arxiv.org/abs/2207.12598}, 
}

@misc{fid_seminal,
      title={GANs Trained by a Two Time-Scale Update Rule Converge to a Local Nash Equilibrium}, 
      author={Martin Heusel and Hubert Ramsauer and Thomas Unterthiner and Bernhard Nessler and Sepp Hochreiter},
      year={2018},
      eprint={1706.08500},
      archivePrefix={arXiv},
      primaryClass={cs.LG},
      url={https://arxiv.org/abs/1706.08500}, 
}

@misc{IS_seminal,
      title={Improved Techniques for Training GANs}, 
      author={Tim Salimans and Ian Goodfellow and Wojciech Zaremba and Vicki Cheung and Alec Radford and Xi Chen},
      year={2016},
      eprint={1606.03498},
      archivePrefix={arXiv},
      primaryClass={cs.LG},
      url={https://arxiv.org/abs/1606.03498}, 
}

@misc{FlowDCN,
      title={FlowDCN: Exploring DCN-like Architectures for Fast Image Generation with Arbitrary Resolution}, 
      author={Shuai Wang and Zexian Li and Tianhui Song and Xubin Li and Tiezheng Ge and Bo Zheng and Limin Wang},
      year={2024},
      eprint={2410.22655},
      archivePrefix={arXiv},
      primaryClass={cs.CV},
      url={https://arxiv.org/abs/2410.22655}, 
}

\newpage
\appendix

\section{Notation guide}
\label{app:notation-guide}

Table~\ref{tab:notation-guide} summarizes the main path, marginal, and velocity notation used throughout the paper. The generic notation $\gamma$ and $v_\gamma^\star$ is used in the population flow-matching statements; the symbols $\gamma_\psi$ and $\hat\gamma_\psi$ denote the two learned-path special cases used by the algorithm and the stochastic regularizer.

\begin{table}[h]
\centering
\caption{Notation convention for paths, induced marginals, and velocity fields.}

\begin{tabular}{lp{0.70\linewidth}}
\toprule
Symbol & Meaning \\
\midrule
$\gamma_{\lin}$ & Standard linear path, $\gamma_{\lin}(t;x_0,x_1)=(1-t)x_0+tx_1$. \\
$\gamma$ & Generic sufficiently regular path, possibly with auxiliary randomness $\xi$. \\
$\gamma_\psi$ & Unregularized learned residual path. \\
$\hat\gamma_\psi$ & Stochastic-regularized learned residual path. \\
$p_{\gamma,t}$ & Marginal distribution of a generic path $\gamma(t;x_0,x_1,y,\xi)$ at time $t$. \\
$p_{\psi,t}$ & Marginal distribution of the unregularized learned path $\gamma_\psi(t;x_0,x_1,y)$. \\
$\hat p_{\psi,t}$ & Marginal distribution of the stochastic-regularized path $\hat\gamma_\psi(t;x_0,x_1,y,\epsilon)$. \\
$p_{\phi,t}$ & ODE marginal induced by the learned sampler $v_\phi$; this is not a path marginal unless $v_\phi$ equals the corresponding ideal path velocity. \\
$v_\gamma^\star$ & Population flow-matching target for a generic path $\gamma$; it is the conditional mean of the path tangent given position, time, and label. \\
$v_\psi^\star$ & Population flow-matching target for the unregularized learned path $\gamma_\psi$. \\
$\hat v_\psi^\star$ & Population flow-matching target for the stochastic-regularized path $\hat\gamma_\psi$. \\
$v_\phi$ & Trained flow network used for sampling at inference time. \\
$u_t$ & Generic velocity field used in entropy-rate identities. \\
$\rho_t$ & Time-dependent source-perturbation schedule. \\
$\sigma_t$ & Gaussian smoothing scale induced by source perturbation, $\sigma_t=(1-t)\sqrt{1-\rho_t^2}$. \\
$L_{\euc}^{\rho,\psi}(t,y)$ & Pointwise kinetic energy of the stochastic-regularized path, defined in \eqref{eq:pointwise-leuc-rho}. \\
\bottomrule
\end{tabular}
\label{tab:notation-guide}
\end{table}

\section{Discussion and Limitations}
\label{sec:discussion-limitations}

Our experiments focus on the image generation task using FID as the key metric. While this is a challenging problem, flow-based generative models have been applied in a broad range of settings. It would be useful to investigate the efficacy of our method in these applications as well.

Path--flow alignment requires additional training compute and memory for the path network, and alternating path and flow optimization adds engineering complexity. The compute-matched results in Appendix~\ref{app:compute-matched-continuation} show that the gains persist under comparable continuation compute, but do not remove this practical overhead. Our experiments study continued training from pretrained or converged checkpoints rather than joint training from scratch. Although Appendix~\ref{app:flowdcn-transfer} broadens the architecture and dataset scope, the evaluation remains restricted to class-conditional image generation; text-to-image models and non-image domains remain untested. The sensitivity study over $\rho_t$ and $\beta$ is limited, the path network and subtractor design is not exhaustively ablated, and the multi-seed results in Appendix~\ref{app:generation-seed-robustness} measure generation-seed variability rather than variability across independently trained models.

Theorem~\ref{thm:source-randomization-smoothing} and Proposition~\ref{prop:kinetic-entropy-rate-bound} control the stochastic training-path marginals and their path-induced population velocity, not directly the marginals produced by the finite learned sampler. These marginals coincide when the learned flow equals the path-induced population velocity, but may differ under finite optimization and model capacity; velocity-level flow-matching control alone does not provide the derivative-level control needed to transfer the entropy guarantee. Accordingly, the observed association between late-time divergence, entropy bottlenecks, and FID is empirical rather than causal. Finally, the Gaussian-convolution, entropy, and Fisher-information guarantees require the stated conditional-independence condition and do not automatically extend to general endpoint couplings such as optimal transport.

\section{Path network and subtractor implementation}
\label{app:path-subtractor-implementation}

Below we describe the implementation of the residual path network $R_{\psi}$ used in the SiT experiments. The final sampler is still an ordinary SiT velocity model $v_\phi$. The additional path machinery consists of two SiT-shaped networks used only to define the training path: a trainable path network $P_\psi$ and a frozen subtractor network $S$. The role of the subtractor is to make the learned path start exactly as the linear path while still allowing $P_\psi$ to learn endpoint-conditioned residual curvature.

\paragraph{Teacher-residualized path parameterization.}
Let
\[
    \gamma_{\lin}(t;x_0,x_1)=(1-t)x_0+tx_1.
\]
The implemented path is
\begin{equation}
\begin{aligned}
    \gamma_\psi^{\mathrm{impl}}(t;x_0,x_1,y,\epsilon_0)
    &=\gamma_{\lin}(t;x_0,x_1) \\
    &\quad + t(1-t)
    \Big[
        P_\psi(\hat\gamma_{\lin}(t),\hat\Delta(t),t,y)
        -S(\hat\gamma_{\lin}(t),t,y)
    \Big].
\end{aligned}
\label{eq:app-implemented-path}
\end{equation}
thus \eqref{eq:app-implemented-path} matches the residual form in \eqref{eq:learned-interpolant} with
\[
    R_\psi(t;x_0,x_1,y)
    =P_\psi(\gamma_{\lin}(t;x_0,x_1),x_1-x_0,t,y)
    -S(\gamma_{\lin}(t;x_0,x_1),t,y).
\]

\paragraph{Perturbed endpoint inputs.}
The outer anchor path in \eqref{eq:app-implemented-path} always uses the true endpoints $(x_0,x_1)$. Perturbations are applied only to the inputs of $P_\psi$ and $S$. Given schedules $\rho_t\in[0,1]$, define
\begin{equation}
    \hat x_0(t)=\rho_t x_0+\sqrt{1-\rho_t^2}\,\epsilon_0.
    \label{eq:app-perturbed-endpoints}
\end{equation}
The source perturbation $\rho_t$ is the stochastic regularizer analyzed in Section~\ref{sec:gaussian-perturbation-definition}. The actual inputs to the path network and subtractor are
\begin{equation}
    \hat\gamma_{\lin}(t)=(1-t)\hat x_0(t)+tx_1,
    \qquad
    \hat\Delta(t)=x_1-\hat x_0(t).
    \label{eq:app-hatted-network-inputs}
\end{equation}
Thus stochastic regularization hides part of the endpoint information from the residual branch, while the endpoint identities of the learned path remain exact because the residual term is multiplied by the boundary envelope.

\paragraph{Subtractor network.}
The subtractor $S(x,t,y)$ is an unmodified SiT velocity network with the standard single patch-embedding stem, timestep embedding, class-label embedding, SiT transformer blocks, and final velocity head. In the main experiments, $S$ is loaded from a pretrained SiT-S/2 checkpoint and frozen. During subsequent flow training, both the learned path network and the subtractor are frozen; together they define the training states and tangent targets for the flow model $v_\phi$.

\paragraph{Dual-stem path network.}
The path network $P_\psi$ uses the same SiT transformer body as $S$, but replaces the single spatial input stem with two stems. The first stem embeds the hatted linear state and the second stem embeds the hatted endpoint displacement:
\[
    z_{\lin}=\mathrm{Embed}_{\lin}(\hat\gamma_{\lin}(t)),
    \qquad
    z_{\Delta}=\mathrm{Embed}_{\Delta}(\hat\Delta(t)).
\]
The token sequence entering the transformer is
\begin{equation}
    z=z_{\lin}+z_{\Delta}+z_{\mathrm{pos}},
    \label{eq:app-dual-stem-tokens}
\end{equation}
where $z_{\mathrm{pos}}$ is the fixed sinusoidal positional embedding used by SiT. The usual SiT conditioning vector is
\[
    c=e_t(t)+e_y(y),
\]
where $e_t$ is the timestep embedding and $e_y$ is the class-label embedding. When endpoint conditioning is enabled, the delta-stem tokens are also average pooled and passed through a small projector,
\begin{equation}
    c=e_t(t)+e_y(y)
    +W_{\mathrm{end}}\,\mathrm{SiLU}\!\left(
        \mathrm{LayerNorm}\!\left(\frac{1}{N}\sum_{i=1}^N z_{\Delta,i}\right)
    \right).
    \label{eq:app-endpoint-conditioner}
\end{equation}
This global endpoint-conditioning branch is the \texttt{EndpointConditionProjector} module, and the full path network is implemented as \texttt{DualStemPathSiT}. The transformer blocks and final head are otherwise the same as in the base SiT velocity model.

\paragraph{Initialization and exact linear startup.}
The path network $P_\psi$ and subtractor $S$ are initialized from the same pretrained SiT checkpoint. Loading the checkpoint into $P_\psi$ proceeds as follows. The vanilla SiT input stem is copied into the linear stem $\mathrm{Embed}_{\lin}$. The timestep embedder, label embedder, positional embedding, transformer blocks, and final layer are copied unchanged. The new delta stem $\mathrm{Embed}_{\Delta}$ is zero-initialized, and the final linear layer of the endpoint-conditioning projector is also zero-initialized. Consequently, at initialization,
\begin{equation}
    P_\psi(\hat\gamma_{\lin}(t),\hat\Delta(t),t,y)
    =S(\hat\gamma_{\lin}(t),t,y),
    \label{eq:app-path-subtractor-cancel}
\end{equation}
provided both networks are loaded from the same checkpoint. Therefore the bracketed residual in \eqref{eq:app-implemented-path} is exactly zero and
\[
    \gamma_\psi^{\mathrm{impl}}(t;x_0,x_1,y,\epsilon_0)=\gamma_{\lin}(t;x_0,x_1)
\]
at the start of path training. This exact startup property is why we use the path network and subtractor pair rather than directly learning an unconstrained displacement from scratch.

\section{Diagnostic correlations for path overfitting}
\label{app:fid-diagnostic-correlations}

In Figures \ref{fig:L_fm_overfitting}--\ref{fig:fid_overfitting} and \ref{fig:app-fid-last-bucket-divergence}--\ref{fig:app-fid-path-geometry}, we use the same settings as SiT-S in Table~\ref{tab:model_comparison} other than changes to $\beta$ and $\rho$ noted in each figure's legends. We provide more detailed descriptions for Figures \ref{fig:L_fm_overfitting}--\ref{fig:fid_overfitting}: curves with $\rho_t=1$ remove stochastic source regularization and vary kinetic regularization through $\beta$, while $\rho_t=0.5$ and $\rho_t=1-t$ are stochastic variants at $\beta=0.05$. Stronger kinetic regularization alone reduces (but does not prevent) late degradation.

This appendix provides the scatter-plot analysis behind the diagnostic discussion in Sections~\ref{sec:path-overfitting} and~\ref{sec:divergence-entropy}. All points are SiT-S runs evaluated with 50k FID. Each plot contains checkpoints after 4k, 8k, and 12k flow-training steps, with marker shape indicating the step and color indicating the algorithmic variant. The variants sweep the kinetic energy regularization strength \(\beta\), the source-perturbation schedule \(\rho_t\), and an additional acceleration-regularized ablation. The acceleration-regularized run adds
\[
    \lambda\,\E\int_0^1
    \norm{\partial_t^2\gamma_\psi(t;x_0,x_1,y)}_2^2\,dt
\]
to the path objective, with \(\lambda=2\times 10^{-5}\). Larger values of \(\lambda\) performed poorly in our preliminary sweeps and are not included here. The vertical axis is truncated to the FID range \([8,11]\), which is the range of interest; higher FID values correspond to extremely bad path-overfitting runs.

We emphasize these plots as qualitative diagnostics rather than as a statistical correlation study. The number of runs is small and the points are coupled through common checkpoints and training schedules. Nevertheless, the comparison is useful: among the measured quantities, the last-bucket divergence gives the clearest visual correlation with FID.

\begin{figure}[t]
\centering
\includegraphics[width=0.72\linewidth]{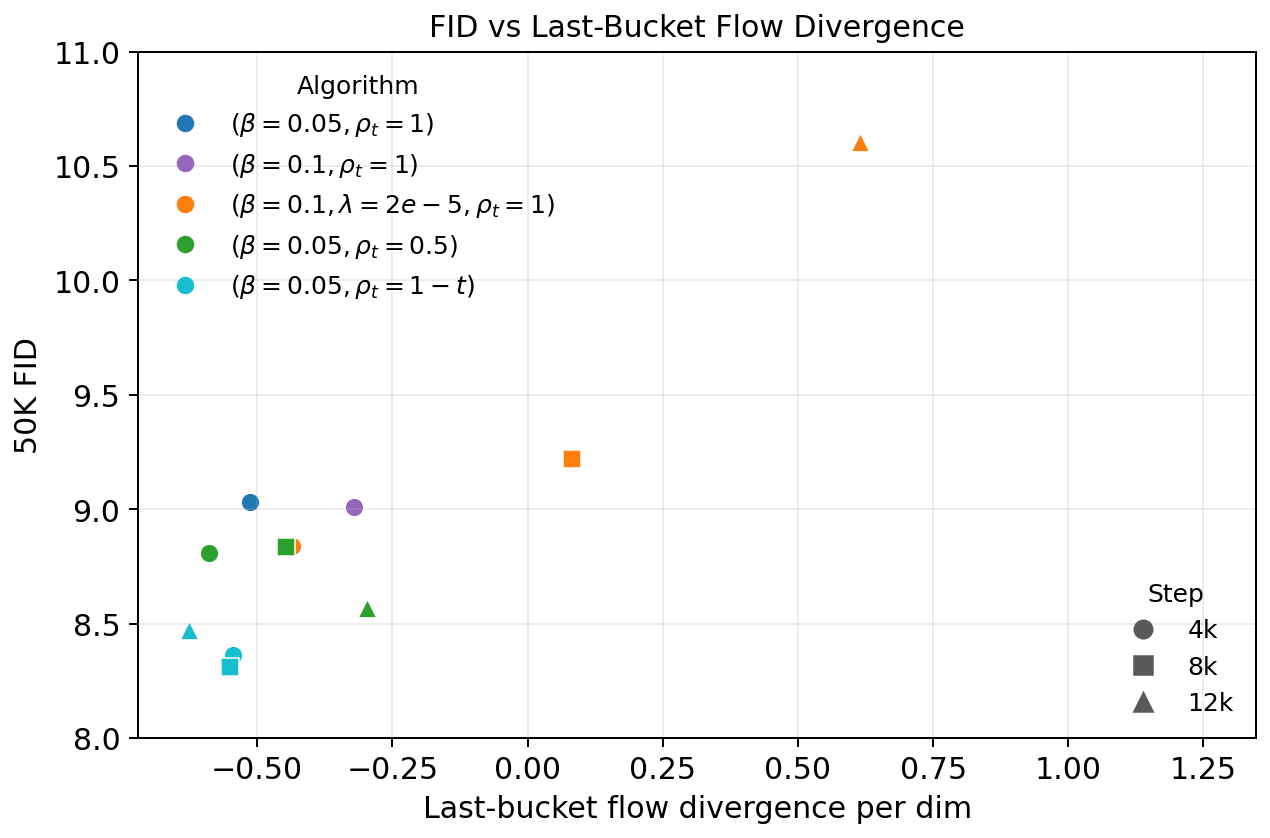}
\caption{FID versus per-dimensional flow divergence in the last time bucket. The last bucket corresponds to the final bin of the ten-bin estimator in \eqref{eq:binned-divergence-diagnostic}. Runs with large positive last-bucket divergence are precisely the high-FID overfit checkpoints in this sweep, while the best checkpoints keep the final-bin divergence negative or only mildly positive.}
\label{fig:app-fid-last-bucket-divergence}
\end{figure}

Figure~\ref{fig:app-fid-last-bucket-divergence} shows the strongest qualitative relationship. The unregularized or weakly regularized checkpoints that degrade in FID develop a large positive final-bin divergence, consistent with the entropy-bottleneck picture in Section~\ref{sec:divergence-entropy}. In contrast, the stochastic source-regularized checkpoints maintain low FID while avoiding a large positive late-time divergence spike. The acceleration-regularized baseline can reduce some geometric quantities, but it does not reliably remove this late positive-divergence behavior.

\begin{figure}[t]
\centering
\begin{minipage}[t]{0.48\textwidth}
    \centering
    \includegraphics[width=\linewidth]{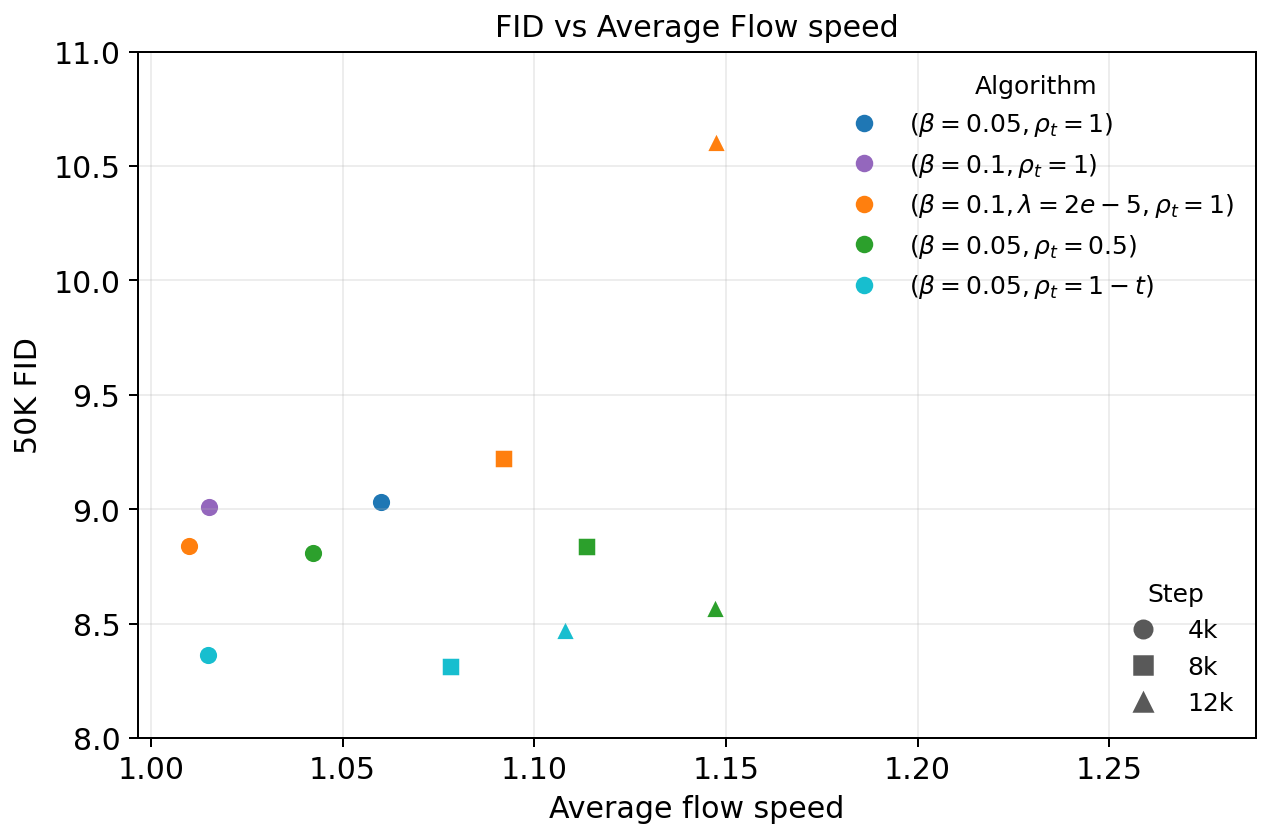}
    \centerline{\small (a) Average flow Euclidean norm.}
\end{minipage}
\hfill
\begin{minipage}[t]{0.48\textwidth}
    \centering
    \includegraphics[width=\linewidth]{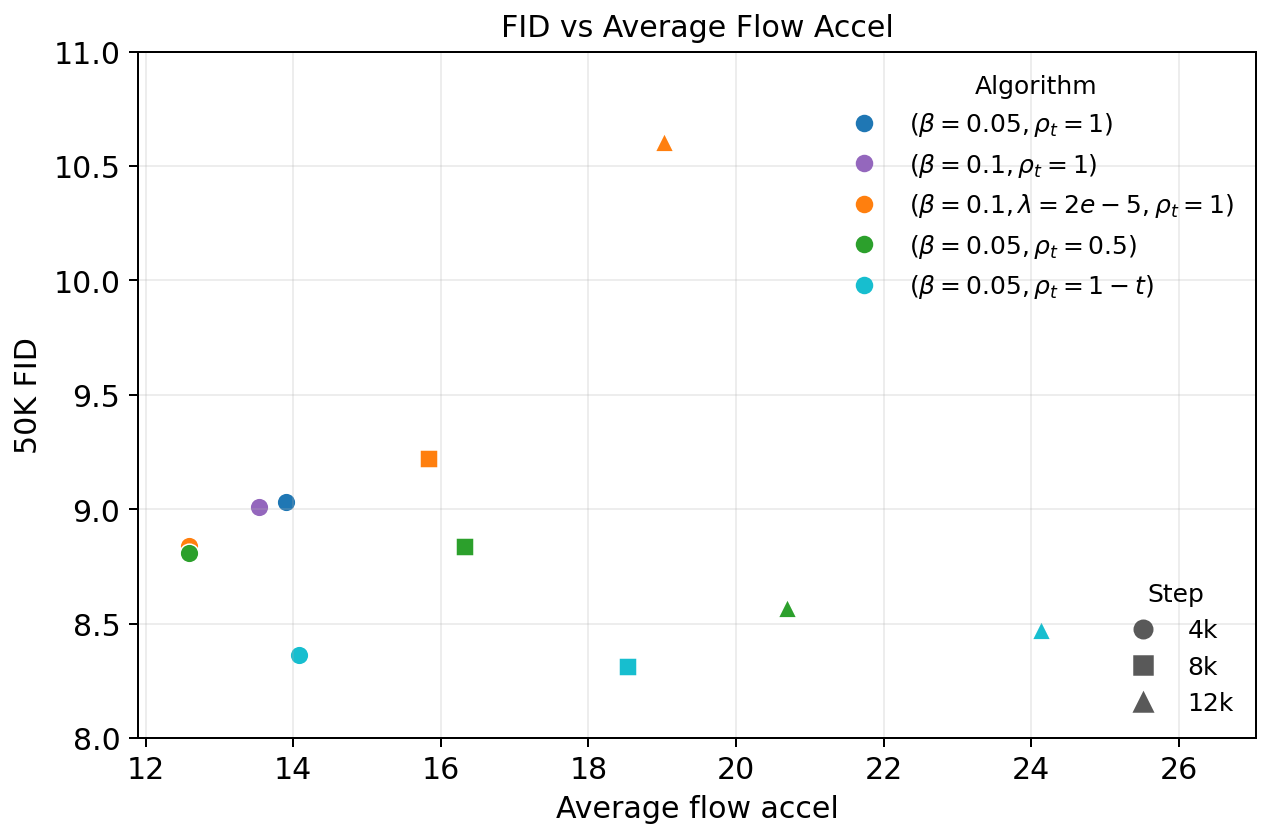}
    \centerline{\small (b) Average flow acceleration.}
\end{minipage}
\caption{FID versus global flow-geometry summaries. Unlike the last-bucket divergence in Figure~\ref{fig:app-fid-last-bucket-divergence}, average flow speed and acceleration are not monotone indicators of sample quality. Some low-FID stochastic source-regularized checkpoints have comparable or larger average flow speed and acceleration than higher-FID unregularized checkpoints.}
\label{fig:app-fid-flow-geometry}
\end{figure}

\begin{figure}[t]
\centering
\begin{minipage}[t]{0.48\textwidth}
    \centering
    \includegraphics[width=\linewidth]{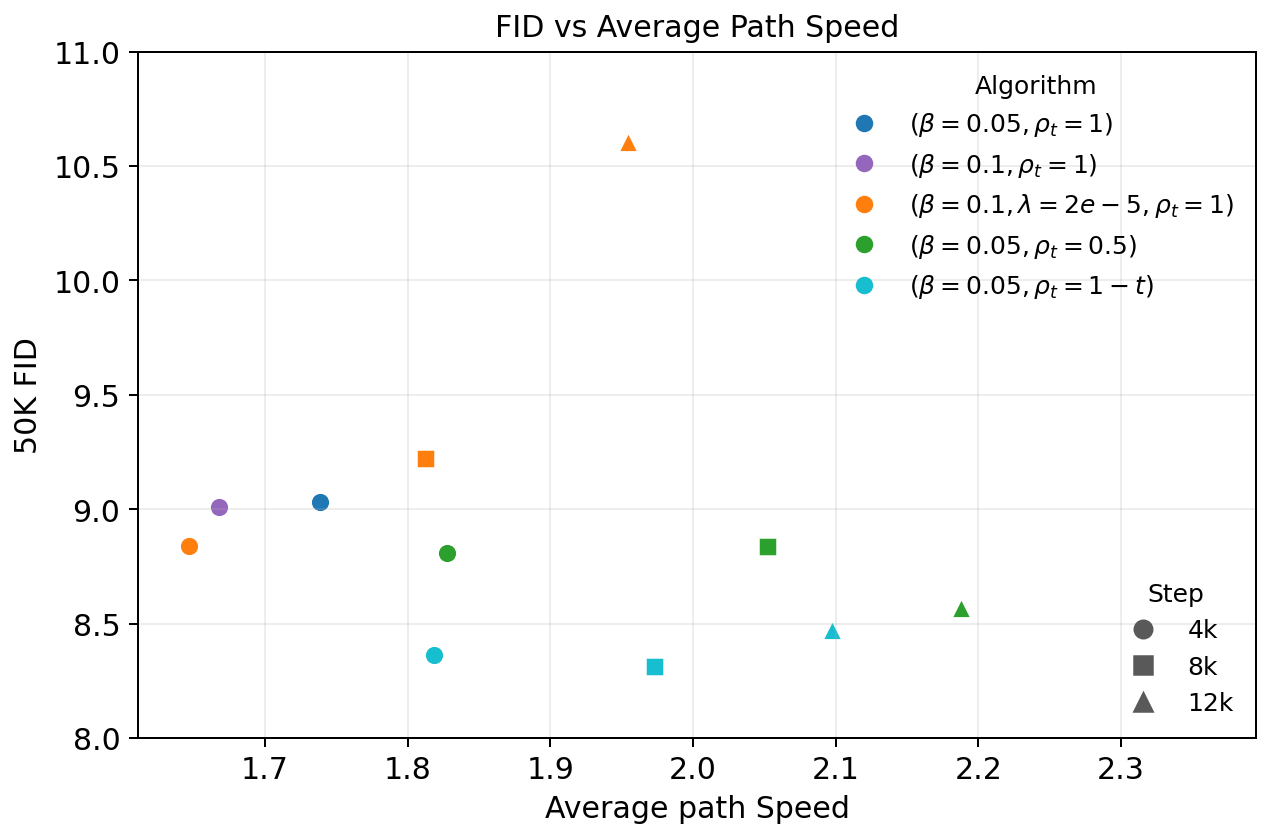}
    \centerline{\small (a) Average path Euclidean norm.}
\end{minipage}
\hfill
\begin{minipage}[t]{0.48\textwidth}
    \centering
    \includegraphics[width=\linewidth]{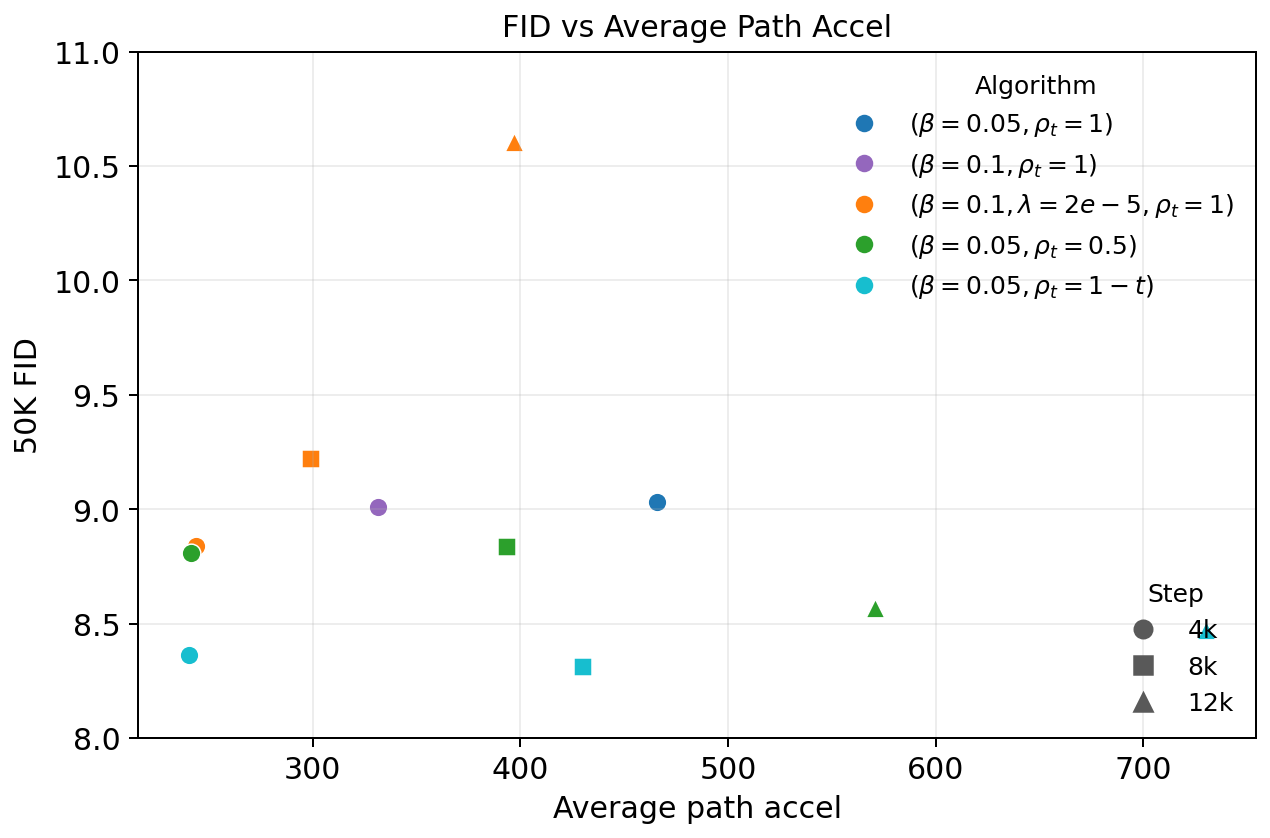}
    \centerline{\small (b) Average path acceleration.}
\end{minipage}
\caption{FID versus global path-geometry summaries. The path-speed and path-acceleration summaries also do not explain the FID ranking on their own. In particular, good stochastic source-regularized checkpoints can have larger path speed or acceleration than worse unregularized checkpoints. This supports the view that the key failure mode is distributional---a late entropy-expansion spike---rather than simply excessive path length or curvature.}
\label{fig:app-fid-path-geometry}
\end{figure}

Taken together, Figures~\ref{fig:app-fid-last-bucket-divergence}--\ref{fig:app-fid-path-geometry} highlight that the divergence spike appears to be a stronger predictor of path overfitting, compared to other geometric quantities such as speed and acceleration of the path or the flow. The late-time divergence statistic is more directly tied to the entropy-rate identity of Proposition~\ref{prop:entropy-rate}, and in this sweep it is the clearest diagnostic of FID degradation.

\section{Proofs for Key Results}
\subsection{Proof of Proposition \ref{prop:fm-projection}}
\label{ss:proof:prop:fm-projection}

\begin{proof}[Proof of Proposition~\ref{prop:fm-projection}]
We first prove the population regression statement. All expectations are taken with respect to the chosen endpoint coupling conditional on $y$, the auxiliary randomness $\xi$, and $t\sim\Unif[0,1]$. Let
\[
    \mathcal G
    =
    \sigma\!\left(\gamma(t;x_0,x_1,y,\xi),t,y\right)
\]
be the $\sigma$-algebra generated by the point on the path, the time, and the label. For any square-integrable vector field $v_\phi$, the random variable
\[
    v_\phi(\gamma(t;x_0,x_1,y,\xi),t,y)
\]
is $\mathcal G$-measurable. By the tower property,
\begin{equation}
\begin{aligned}
    \mathcal L_{\FM}(v_\phi\mid\gamma)
    &=
    \E\Big[
        \norm{v_\phi(\gamma(t;x_0,x_1,y,\xi),t,y)
        -\dot\gamma(t;x_0,x_1,y,\xi)}_2^2
    \Big] \\
    &=
    \E\Big[
        \E\big[
        \norm{v_\phi(\gamma(t;x_0,x_1,y,\xi),t,y)
        -\dot\gamma(t;x_0,x_1,y,\xi)}_2^2
        \,\big|\,
        \mathcal G
        \big]
    \Big].
\end{aligned}
\label{eq:conditional-risk-tower}
\end{equation}

Fix a realization of $\mathcal G$, say
\[
    \gamma(t;x_0,x_1,y,\xi)=x,
    \qquad t=s,
    \qquad y=\bar y.
\]
Conditional on this event, the value of the vector field is a deterministic vector $a=v_\phi(x,s,\bar y)$. The conditional risk is
\[
    \E\!\left[
        \norm{a-\dot\gamma(t;x_0,x_1,y,\xi)}_2^2
        \;\middle|\;
        \gamma(t;x_0,x_1,y,\xi)=x,\,t=s,\,y=\bar y
    \right].
\]
Let
\[
    m(x,s,\bar y)
    =
    \E\!\left[
        \dot\gamma(t;x_0,x_1,y,\xi)
        \;\middle|\;
        \gamma(t;x_0,x_1,y,\xi)=x,\,t=s,\,y=\bar y
    \right].
\]
Then
\begin{align*}
    &\E\!\left[
        \norm{a-\dot\gamma(t;x_0,x_1,y,\xi)}_2^2
        \;\middle|\;
        \gamma(t;x_0,x_1,y,\xi)=x,\,t=s,\,y=\bar y
    \right] \\
    &\qquad =
    \norm{a-m(x,s,\bar y)}_2^2
    +
    \E\!\left[
        \norm{\dot\gamma(t;x_0,x_1,y,\xi)-m(x,s,\bar y)}_2^2
        \;\middle|\;
        \gamma(t;x_0,x_1,y,\xi)=x,\,t=s,\,y=\bar y
    \right],
\end{align*}
because the conditional mean of
$\dot\gamma(t;x_0,x_1,y,\xi)-m(x,s,\bar y)$ is zero under the same conditioning. The second term does not depend on $a$, while the first term is uniquely minimized by $a=m(x,s,\bar y)$. Therefore the population minimizer is
\[
    v_\gamma^\star(x,t,y)
    =
    \E\!\left[
        \dot\gamma(t;x_0,x_1,y,\xi)
        \;\middle|\;
        \gamma(t;x_0,x_1,y,\xi)=x,\,t,\,y
    \right]
\]
for $p_{\gamma,t}(\cdot\mid y)\,dt$-almost every $(x,t)$ and for the relevant labels $y$. This proves the first part of the proposition. Notice that no endpoint-preservation assumption was used.

The same conditional-expectation argument also gives the projection identity
\begin{equation}
\begin{aligned}
    \mathcal L_{\FM}(v_\phi\mid\gamma)
    &=
    \E\Big[
        \norm{
        v_\phi(\gamma(t;x_0,x_1,y,\xi),t,y)
        -v_\gamma^\star(\gamma(t;x_0,x_1,y,\xi),t,y)
        }_2^2
    \Big] \\
    &\quad+
    \E\Big[
        \norm{
        \dot\gamma(t;x_0,x_1,y,\xi)
        -v_\gamma^\star(\gamma(t;x_0,x_1,y,\xi),t,y)
        }_2^2
    \Big].
\end{aligned}
\label{eq:fm-projection}
\end{equation}
Indeed, expand
\[
\begin{aligned}
    &\dot\gamma(t;x_0,x_1,y,\xi)
    -v_\phi(\gamma(t;x_0,x_1,y,\xi),t,y) \\
    &\qquad =
    \big(
        \dot\gamma(t;x_0,x_1,y,\xi)
        -v_\gamma^\star(\gamma(t;x_0,x_1,y,\xi),t,y)
    \big) \\
    &\qquad\quad+
    \big(
        v_\gamma^\star(\gamma(t;x_0,x_1,y,\xi),t,y)
        -v_\phi(\gamma(t;x_0,x_1,y,\xi),t,y)
    \big).
\end{aligned}
\]
The cross term has expectation zero because the second factor is $\mathcal G$-measurable and
\[
    \E\!\left[
        \dot\gamma(t;x_0,x_1,y,\xi)
        -v_\gamma^\star(\gamma(t;x_0,x_1,y,\xi),t,y)
        \mid \mathcal G
    \right]=0.
\]
This proves \eqref{eq:fm-projection}.

We now prove the ODE transport statement. Fix $y$ and assume the additional hypotheses in the second part of Proposition~\ref{prop:fm-projection}. In particular, the Borel, Lipschitz, and linear-growth conditions imply that $v_\gamma^\star(\cdot,\cdot,y)$ generates a unique Carath\'eodory flow. Let $f\in C_c^\infty(\R^d)$ be a test function. Differentiating under the expectation gives
\[
    \frac{d}{dt}\E[f(\gamma(t;x_0,x_1,y,\xi))\mid y]
    =
    \E[\inner{\grad f(\gamma(t;x_0,x_1,y,\xi))}
    {\dot\gamma(t;x_0,x_1,y,\xi)}\mid y].
\]
Conditioning on $\gamma(t;x_0,x_1,y,\xi)=x$ and using the definition of $v_\gamma^\star$,
\[
    \frac{d}{dt}\int f(x)p_{\gamma,t}(x\mid y)\,dx
    =
    \int \inner{\grad f(x)}{v_\gamma^\star(x,t,y)}p_{\gamma,t}(x\mid y)\,dx.
\]
This is the weak form of the continuity equation associated with the velocity field $v_\gamma^\star$. The law of the Carath\'eodory flow
\[
    \frac{d}{dt}z(t)=v_\gamma^\star(z(t),t,y)
\]
satisfies the same continuity equation. By endpoint preservation and $x_0\sim p_0$, the path marginal at $t=0$ is $p_0$, so the ODE and the path have the same initial law. Under the Lipschitz and linear-growth assumptions of Proposition~\ref{prop:fm-projection}, the associated probability solution is unique. Hence the ODE marginals coincide with the path marginals for every $t\in[0,1]$. Finally, endpoint preservation and $x_1\sim p_1(\cdot\mid y)$ imply that the path marginal at $t=1$ is $p_1(\cdot\mid y)$, so $z(1)\sim p_1(\cdot\mid y)$.
\end{proof}

\subsection{Proof of Proposition \ref{prop:entropy-rate}}
\label{ss:proof:prop:entropy-rate}

\begin{proof}
The differential entropy is
\[
    h(p_t(\cdot\mid y))=-\int p_t(x\mid y)\log p_t(x\mid y)\,dx.
\]
Differentiating and using $\int \partial_t p_t(x\mid y)\,dx=0$ gives
\[
    \frac{d}{dt}h(p_t(\cdot\mid y))
    =
    -\int \partial_t p_t(x\mid y)\log p_t(x\mid y)\,dx.
\]
Substituting $\partial_t p_t(x\mid y)
    +
    \divx\!\left(p_t(x\mid y)u_t(x,y)\right)
    =0,$ yields
\[
    \frac{d}{dt}h(p_t(\cdot\mid y))
    =
    \int \divx(p_t u_t)\log p_t\,dx.
\]
Integrating by parts and using the decay assumption,
\[
    \frac{d}{dt}h(p_t(\cdot\mid y))
    =
    -\int p_t(x\mid y)u_t(x,y)\cdot \grad\log p_t(x\mid y)\,dx
    =
    \int p_t(x\mid y)\divx u_t(x,y)\,dx.
\]
This proves \eqref{eq:generic-entropy-rate}. Integrating over $[a,b]$ proves \eqref{eq:generic-binned-entropy-divergence}.
\end{proof}

\subsection{Proof of Theorem \ref{thm:source-randomization-smoothing}}
\label{sec:proof:thm:source-randomization-smoothing}
\begin{proof}
For the fixed time $t$, the pair $(x_0,\hat x_0(t))$ is jointly Gaussian and satisfies
\[
    \mathrm{Cov}(x_0)=I_d,
    \qquad
    \mathrm{Cov}(\hat x_0(t))=I_d,
    \qquad
    \mathrm{Cov}(x_0,\hat x_0(t))=\rho_t I_d.
\]
The variable $\eta_t$ in \eqref{eq:eta-t-definition} is therefore jointly Gaussian with $\hat x_0(t)$. A direct covariance calculation gives
\[
    \mathrm{Cov}(\eta_t)=I_d,
    \qquad
    \mathrm{Cov}(\eta_t,\hat x_0(t))=0.
\]
Hence, conditional on $y$, $\eta_t\sim\Ncal(0,I_d)$ and $\eta_t$ is independent of $\hat x_0(t)$. Under the independent endpoint coupling, $x_1$ is independent of $(x_0,\epsilon)$ conditional on $y$, so $\eta_t$ is conditionally independent of $(\hat x_0(t),x_1)$ given $y$. Rearranging \eqref{eq:eta-t-definition} gives
\[
    x_0=\rho_t\hat x_0(t)+\sqrt{1-\rho_t^2}\,\eta_t.
\]
Substituting this identity into \eqref{eq:time-dependent-rho-interpolant} yields \eqref{eq:rho-gamma-convolution-form}. Since the last term $\sigma_t\eta_t$ is Gaussian and independent of the non-Gaussian term \eqref{eq:qrho-pushforward-variable}, the conditional law of their sum is the convolution \eqref{eq:rho-convolution}.

The entropy floor follows by conditioning on the non-Gaussian term in \eqref{eq:qrho-pushforward-variable}. After this conditioning, the only remaining randomness in the state is the independent Gaussian $\sigma_t\eta_t$, whose entropy is
\[
    \frac d2\log(2\pi e\sigma_t^2).
\]
Removing the conditioning can only increase entropy by a nonnegative mutual-information term, giving \eqref{eq:entropy-floor}.

It remains to explain the Fisher-information bound. Let $\varphi_{\sigma_t}$ denote the density of $\Ncal(0,\sigma_t^2I_d)$. By the convolution representation,
\[
    \hat p_{\psi,t}(x\mid y)
    =\int \varphi_{\sigma_t}(x-r)\,dq_{\psi,t}(r\mid y).
\]
Differentiating under the integral gives
\[
    \grad \hat p_{\psi,t}(x\mid y)
    =-\frac{1}{\sigma_t^2}
      \int (x-r)\varphi_{\sigma_t}(x-r)\,dq_{\psi,t}(r\mid y).
\]
Equivalently, using the representation
$\hat\gamma_\psi(t)=r+\sigma_t\eta_t$ from \eqref{eq:rho-gamma-convolution-form},
\begin{equation}
    \grad\log\hat p_{\psi,t}(x\mid y)
    =-\frac{1}{\sigma_t}
      \E\left[\eta_t\mid \hat\gamma_\psi(t;x_0,x_1,y,\epsilon)=x,\,y\right].
    \label{eq:gaussian-posterior-score-identity}
\end{equation}
Therefore Jensen's inequality gives
\[
\begin{aligned}
    I(\hat p_{\psi,t}(\cdot\mid y))
    &=\frac{1}{\sigma_t^2}
      \E\left[
      \norm{\E[\eta_t\mid \hat\gamma_\psi(t),y]}_2^2
      \mid y\right] \\
    &\le
      \frac{1}{\sigma_t^2}\E[\norm{\eta_t}_2^2]
      =\frac{d}{\sigma_t^2}.
\end{aligned}
\]
This proves \eqref{eq:fisher-gaussian-smoothing-bound}.
\end{proof}

\subsection{Proof of Proposition \ref{prop:kinetic-entropy-rate-bound}}
\label{sec:proof:prop:kinetic-entropy-rate-bound}
\begin{proof}
By the definition of the stochastic-path population target in \eqref{eq:rho-ideal-flow}, the ideal marginal velocity is
\[
    \hat v_\psi^\star(x,t,y)
    =
    \E\left[
        \partial_t\hat\gamma_\psi(t;x_0,x_1,y,\epsilon)
        \;\middle|\;
        \hat\gamma_\psi(t;x_0,x_1,y,\epsilon)=x,\,t,\,y
    \right].
\]
Applying Proposition~\ref{prop:entropy-rate} to the probability flow generated by $\hat v_\psi^\star$ gives
\[
    \frac{d}{dt}h(\hat p_{\psi,t}(\cdot\mid y))
    =
    \E_{x\sim\hat p_{\psi,t}(\cdot\mid y)}
    [\divx \hat v_\psi^\star(x,t,y)].
\]
Integrating by parts,
\[
\begin{aligned}
    \frac{d}{dt}h(\hat p_{\psi,t}(\cdot\mid y))
    &=-
    \E_{x\sim\hat p_{\psi,t}(\cdot\mid y)}
    \left[
        \inner{\hat v_\psi^\star(x,t,y)}
        {\grad\log\hat p_{\psi,t}(x\mid y)}
    \right] \\
    &\ge
    -
    \sqrt{
        \E_{x\sim\hat p_{\psi,t}(\cdot\mid y)}
        \norm{\hat v_\psi^\star(x,t,y)}_2^2
    }
    \sqrt{I(\hat p_{\psi,t}(\cdot\mid y))},
\end{aligned}
\]
where the inequality is Cauchy--Schwarz.

The first factor is controlled by the path kinetic energy. Indeed, by Jensen's inequality for conditional expectation,
\[
\begin{aligned}
    \E_{x\sim\hat p_{\psi,t}(\cdot\mid y)}
    \norm{\hat v_\psi^\star(x,t,y)}_2^2
    &\le
    \E\left[
        \norm{\partial_t\hat\gamma_\psi(t;x_0,x_1,y,\epsilon)}_2^2
        \;\middle|\; y
    \right] \\
    &=L_{\euc}^{\rho,\psi}(t,y).
\end{aligned}
\]
The second factor is controlled by Gaussian smoothing: Theorem~\ref{thm:source-randomization-smoothing} gives
\[
    I(\hat p_{\psi,t}(\cdot\mid y))
    \le
    \frac{d}{(1-t)^2(1-\rho_t^2)}.
\]
Combining the last three displays proves \eqref{eq:entropy-rate-kinetic-bound}.
\end{proof}

\section{Experimental details}
\label{sec:experimental_details_appx}

We implement Algorithm~\ref{alg:simple_alternating_minimization} with flow $v_\phi$ initialized from already-trained SiT~\citep{ma2024sit} checkpoints of the appropriate size and path ($P_\psi$ and $S$) initialized from pretrained checkpoints. Our main experiments use ImageNet-$256\times256$ \citep{imagenet_dataset}. Our setting is post-training of a pretrained flow model. For SiT-S/B/L, we finetune the pretrained checkpoints until FID plateaus and use the resulting checkpoint to initialize both the standard continuation baseline and our method. For SiT-XL, finetuning does not improve the pretrained checkpoint FID, so the pretrained checkpoint is the initialization for the flow network $v_\phi$.

Our path network $R_\psi$ is constructed from a pre-trained flow network, as described in Appendix \ref{app:path-subtractor-implementation}. For the FM experiments in Table~\ref{tab:model_comparison}, we use the pretrained SiT-S checkpoint to initialize $P_\psi$ and $S$. For the MG experiments in Table~\ref{tab:fid_50k_mg}, we use the pretrained SiT-B-MG checkpoint to initialize $P_\psi$ and $S$.

\paragraph{Experimental details for FM experiments in Table \ref{tab:model_comparison}.} In each cycle we finetune path $R_\psi$ using maximum learning rate $3\times 10^{-4}$ with cosine decay schedule for 3000 steps, batch size 280, $\beta=0.05$, and $\rho_t = c\cdot (1-t)$ where $c=1$ for SiT-S,B,L and $c=0.7$ for SiT-XL. In each cycle we finetune flow $v_\phi$ using constant learning rate $5\times 10^{-6}$ for 4000 steps with no EMA, batch size 768, and inherit the same $\rho_t$ from the path tuning stage. To ensure a fair comparison, we apply standard finetuning to the pretrained baseline SiT models using both the default pretraining settings reported in SiT \citep{ma2024sit} and our settings for finetuning the flow network $v_\phi$, and report the best FID for the baseline given a fixed finetuning budget. Applying standard finetuning to SiT-S,B,L gives improvements to the FID, whereas finetuning SiT-XL results in worse FID. Experiments are conducted on H200 GPUs. We compute FID~\citep{fid_seminal} and Inception-Score (IS)~\citep{IS_seminal} on 50k images for each checkpoint using the Heun ODE sampler at 250 NFE. We tune the classifier-free guidance strength (CFG)~\citep{cfg_seminal} separately for the SiT baseline and our method, generally noticing higher CFG compared to the baseline working better for Algorithm~\ref{alg:simple_alternating_minimization}. We use CFG $=\{3, 2, 1.5, 1.5\}$ for baseline SiT-S,B,L,XL models, and CFG = $\{3.5, 2.5, 2.0, 1.7\}$ for our SiT-S,B,L,XL models. To ensure a fair comparison, we sweep CFG on the same grid for both the SiT baseline and our method, and report FID with best CFG.

\paragraph{Experimental details for MG experiments in Table~\ref{tab:fid_50k_mg}.}
For the model-guidance experiments, Algorithm~\ref{alg:simple_alternating_minimization} uses three alternating cycles, each with 3k path-tuning steps followed by 4k MG-flow-tuning steps. Both SiT-B-MG and SiT-XL-MG use $\beta=0.05$, flow EMA decay $0.999$, and MG-flow tuning with batch size 768 and learning rate $5\times10^{-6}$. For both models, the path network $P_\psi$ and subtractor $S$ are initialized from the pretrained SiT-B-MG checkpoint and use batch size 140. The SiT-B-MG run uses a constant source-perturbation $\rho_t=0.5$ and path learning rate $3\times10^{-4}$, while SiT-XL-MG uses $\rho_t=0.7(1-t)$ and the path learning rate is $5\times10^{-5}$. For a fair comparison, the standard MG fine-tuning baselines use the same batch size and learning rate as MG-flow tuning. All entries are evaluated with 50k generated samples using the Heun ODE sampler with 250 NFE and CFG $=1.0$.

\begin{table}[!h]
\centering
\caption{Full ImageNet-$256\times256$ results corresponding to Table~\ref{tab:model_comparison}. The baseline columns report the initial checkpoint and the best FID reached by standard fine-tuning within the 12k continuation budget. Our method is evaluated after each alternating cycle.}
\label{tab:model_comparison_full}
\resizebox{\linewidth}{!}{
\begin{tabular}{lcccccccccc}
\toprule
Model & Params & \multicolumn{2}{c}{FID (baseline) $\downarrow$} & \multicolumn{3}{c}{FID (ours) $\downarrow$} & \multicolumn{1}{c}{IS (baseline) $\uparrow$} & \multicolumn{3}{c}{IS (ours) $\uparrow$} \\
\cmidrule(lr){3-4} \cmidrule(lr){5-7} \cmidrule(lr){8-8} \cmidrule(lr){9-11}
 & & Init & Best & 4k & 8k & 12k & Init & 4k & 8k & 12k \\
\midrule
SiT-S  & 33M  & 9.109 & 9.022 & 8.375 & \textbf{8.359} & 8.492  & 173.6 & \textbf{188.9} & 187.2 & 185.4 \\
SiT-B  & 130M & 5.074 & 5.057 & 4.703 & 4.480 & \textbf{4.415} & 199.0 & \textbf{234.9} & 229.0 & 225.5 \\
SiT-L  & 458M & 3.424 & 3.390 & 3.014 & 2.882 & \textbf{2.742} & 201.1 & \textbf{274.8} & 271.5 & 267.2 \\
SiT-XL & 675M & 2.091 & 2.091 & 1.909 & 1.845 & \textbf{1.825} & 256.1 & 282.9 & 285.6 & \textbf{288.6} \\
\bottomrule
\end{tabular}
}
\end{table}

\begin{table}[!h]
\centering
\caption{Checkpoint-wise model-guidance FID results corresponding to Table~\ref{tab:fid_50k_mg}. Both methods start from the same model-guidance checkpoint; lower is better.}
\label{tab:fid_50k_mg_full}
\begin{tabular}{llcccc}
\toprule
Model & Method & 0k & 4k & 8k & 12k \\
\midrule
\multirow{2}{*}{SiT-B-MG (130M)}
  & Ours        & 8.090 & 6.704 & 5.844 & \textbf{5.585} \\
  & MG baseline & 8.090 & 8.011 & 7.940 & 7.950 \\
\midrule
\multirow{2}{*}{SiT-XL-MG (675M)}
  & Ours        & 1.522 & 1.336 & \textbf{1.322} & 1.333 \\
  & MG baseline & 1.522 & 1.494 & 1.504 & 1.501 \\
\bottomrule
\end{tabular}
\end{table}

\begin{table}[!h]
\centering
\caption{Training efficiency measured in steps per second per H200 GPU.}
\label{tab:training-efficiency}
\begin{tabular}{lccc}
\toprule
Model & Alg~\ref{alg:simple_alternating_minimization} (Path Phase)\ & Alg~\ref{alg:simple_alternating_minimization} (Flow Phase) & Standard flow matching\\
\midrule
SiT-S & 1.86 & 1.39 & 4.65 \\
SiT-B & 1.70 & 1.02 & 2.08 \\
SiT-L & 0.85 & 0.51 & 0.71 \\
SiT-XL & 0.70 & 0.37 & 0.50 \\
\bottomrule
\end{tabular}
\end{table}

\paragraph{Efficiency.} Table~\ref{tab:training-efficiency} compares the efficiency of training flow and path networks in Algorithm~\ref{alg:simple_alternating_minimization} to standard flow matching training. The additional cost of using $R_\psi$ during flow training causes our flow training to be less efficient per step. However, the relative cost of $R_\psi$ decreases at larger flow models because $R_\psi$ has constant size across S/B/L/XL. All experiments are conducted on H200 GPUs, and our code supports multi-GPU training and image generation. Appendix~\ref{app:compute-matched-continuation} reports complete GPU-hour- and FLOP-matched experiments.

\subsection{Compute-matched baseline continuation}
\label{app:compute-matched-continuation}

Using our method, three training cycles contain 9k path updates and 12k learned-path flow updates. For SiT-S/B/XL, these cycles cost 1.435/2.546/8.811 EFLOPs, whereas 12k updates of standard continuation cost 0.335/1.272/6.559 EFLOPs. We extend each baseline to the larger of the GPU-hour- and FLOP-matched budgets. GPU-hour matching gives the larger budget for all three models sizes.

\begin{table}[!ht]
\centering
\small
\caption{GPU-hour matching for three training cycles. GPU-hours are measured on H200 GPUs.}
\label{tab:compute-match-gpu-hours}
\begin{tabular}{lccc}
\toprule
Model & Ours GPU-h (path + flow) & Matched baseline steps & Baseline GPU-h \\
\midrule
SiT-S  & 3.574 (1.210 + 2.364) & 61.1k & 3.579 \\
SiT-B  & 4.667 (1.468 + 3.199) & 36.6k & 4.677 \\
SiT-XL & 12.216 (3.546 + 8.670) & 22.8k & 12.243 \\
\bottomrule
\end{tabular}
\end{table}

\begin{table}[!ht]
\centering
\small
\caption{Compute-matched ImageNet-$256\times256$ continuation. FID is computed from 50k images using the Heun sampler at 250 NFE. The reduction is relative to the stronger standard baseline.}
\label{tab:compute-matched-fid}
\begin{tabular}{lccccc}
\toprule
Model & Initialization & Baseline best & Compute-matched & Ours & Reduction \\
\midrule
SiT-S  & 9.109 & 9.022 & 9.025 & \textbf{8.359} & 7.3\% \\
SiT-B  & 5.074 & 5.057 & 4.981 & \textbf{4.415} & 11.4\% \\
SiT-XL & 2.091 & 2.091 & 2.147 & \textbf{1.825} & 12.7\% \\
\bottomrule
\end{tabular}
\end{table}

Longer standard continuation improves only SiT-B beyond its original baseline, while path--flow training gives lower FID for all three model sizes. The inference architecture, NFE, and sampling cost are unchanged. We also evaluated the performance of 12k-step continuation with the original SiT hyperparameters. For SiT-B, our tuned baseline reaches 5.057 FID, compared with 5.533 under the original hyperparameters and 5.074 at initialization. For SiT-XL, the two settings reach 2.134 and 2.095, respectively, compared with 2.091 at initialization.

The cost of path--flow alignment applied to a pretrained checkpoint is small relative to pretraining: SiT-XL/2 pretraining uses an estimated 1275.4 EFLOPs, more than $140\times$ the 8.811 EFLOPs used by three cycles of path--flow training.

\subsection{CFG selection and additional metrics}
\label{app:cfg-and-pr}

We select classifier-free guidance (CFG) independently for each baseline and path--flow checkpoint using 10k generated images and the adaptive Dopri5 sampler. Note that headline evaluations use 50k-FID.

\begin{table}[!ht]
\centering
\scriptsize
\caption{FID on the 10k-image CFG-selection sweeps. Bold and underline identify the selected value for each checkpoint.}
\label{tab:cfg-fid-sweeps}
\begin{minipage}[t]{0.48\linewidth}
\centering
\begin{tabular}{lcccc}
\toprule
\multicolumn{5}{c}{\textbf{SiT-S}} \\
 & CFG 2 & CFG 3 & CFG 3.5 & CFG 4 \\
\midrule
Baseline & 18.742 & \textbf{\underline{13.094}} & 13.099 & 13.967 \\
Ours     & 20.364 & 11.463 & \textbf{\underline{10.988}} & 11.439 \\
\bottomrule
\end{tabular}
\end{minipage}
\hfill
\begin{minipage}[t]{0.48\linewidth}
\centering
\begin{tabular}{lcccc}
\toprule
\multicolumn{5}{c}{\textbf{SiT-B}} \\
 & CFG 1 & CFG 2 & CFG 2.5 & CFG 3 \\
\midrule
Baseline & 31.481 & \textbf{\underline{8.794}} & 9.157 & 11.126 \\
Ours     & 43.778 & 9.576 & \textbf{\underline{7.286}} & 8.051 \\
\bottomrule
\end{tabular}
\end{minipage}

\vspace{4pt}
\begin{minipage}[t]{0.48\linewidth}
\centering
\begin{tabular}{lcccc}
\toprule
\multicolumn{5}{c}{\textbf{SiT-L}} \\
 & CFG 1 & CFG 1.5 & CFG 2 & CFG 2.5 \\
\midrule
Baseline & 17.612 & \textbf{\underline{6.992}} & 7.745 & 10.828 \\
Ours     & 24.455 & 8.050 & \textbf{\underline{5.606}} & 7.666 \\
\bottomrule
\end{tabular}
\end{minipage}
\hfill
\begin{minipage}[t]{0.48\linewidth}
\centering
\begin{tabular}{lcccc}
\toprule
\multicolumn{5}{c}{\textbf{SiT-XL}} \\
 & CFG 1 & CFG 1.5 & CFG 1.7 & CFG 2 \\
\midrule
Baseline & 12.446 & \textbf{\underline{4.977}} & 5.271 & 6.771 \\
Ours     & 14.796 & 5.062 & \textbf{\underline{4.787}} & 5.651 \\
\bottomrule
\end{tabular}
\end{minipage}
\end{table}

The sweeps select CFG 3/2/1.5/1.5 for the SiT-S/B/L/XL baselines and 3.5/2.5/2/1.7 for the corresponding path--flow checkpoints. We use the best CFG settings for each model size and method in our 50k-FID headline results.

In Table~\ref{tab:precision-recall-headline} we additionally report precision and recall corresponding to the SiT checkpoints used for headline results in Table~\ref{tab:model_comparison}. As baseline, we generate using the model with the best FID at every SiT model size. 

\begin{table}[!ht]
\centering
\scriptsize
\caption{Precision and recall corresponding to the 50k-image headline FID results in Table~\ref{tab:fid_50k_mg}, generated using Heun at 250 NFE. Bold marks the better value within each model and metric.}
\label{tab:precision-recall-headline}
\begin{minipage}[t]{0.48\linewidth}
\centering
\begin{tabular}{lcc}
\toprule
\multicolumn{3}{c}{\textbf{SiT-S}} \\
 & Precision $\uparrow$ & Recall $\uparrow$ \\
\midrule
Baseline & \textbf{0.874} & 0.283 \\
Ours     & 0.860 & \textbf{0.297} \\
\bottomrule
\end{tabular}
\end{minipage}
\hfill
\begin{minipage}[t]{0.48\linewidth}
\centering
\begin{tabular}{lcc}
\toprule
\multicolumn{3}{c}{\textbf{SiT-B}} \\
 & Precision $\uparrow$ & Recall $\uparrow$ \\
\midrule
Baseline & \textbf{0.850} & 0.441 \\
Ours     & 0.826 & \textbf{0.474} \\
\bottomrule
\end{tabular}
\end{minipage}

\vspace{4pt}
\begin{minipage}[t]{0.48\linewidth}
\centering
\begin{tabular}{lcc}
\toprule
\multicolumn{3}{c}{\textbf{SiT-L}} \\
 & Precision $\uparrow$ & Recall $\uparrow$ \\
\midrule
Baseline & \textbf{0.819} & 0.542 \\
Ours     & 0.806 & \textbf{0.564} \\
\bottomrule
\end{tabular}
\end{minipage}
\hfill
\begin{minipage}[t]{0.48\linewidth}
\centering
\begin{tabular}{lcc}
\toprule
\multicolumn{3}{c}{\textbf{SiT-XL}} \\
 & Precision $\uparrow$ & Recall $\uparrow$ \\
\midrule
Baseline & \textbf{0.814} & 0.581 \\
Ours     & 0.791 & \textbf{0.611} \\
\bottomrule
\end{tabular}
\end{minipage}
\end{table}

\FloatBarrier 

Recall increases when comparing our method to the baseline, while precision decreases. 

To examine how both metrics vary with guidance, Table~\ref{tab:cfg-pr-sweeps} reports how precision and recall change with respect to CFG using 10k images. Bold marks the best value in each row, and underline marks the CFG selected by FID for the headline evaluation.

\begin{table}[!ht]
\centering
\scriptsize
\caption{Precision and recall across 10k-image CFG sweeps using the Dopri5 sampler.}
\label{tab:cfg-pr-sweeps}
\begin{tabular}{llcccc}
\toprule
\multicolumn{6}{c}{\textbf{SiT-S}} \\
Metric & Method & CFG 2 & CFG 3 & CFG 3.5 & CFG 4 \\
\midrule
Precision $\uparrow$ & Baseline & 0.850 & \textbf{\underline{0.875}} & 0.872 & 0.868 \\
                      & Ours     & 0.831 & \textbf{0.858} & \underline{0.857} & 0.852 \\
Recall $\uparrow$    & Baseline & 0.360 & \textbf{\underline{0.364}} & 0.354 & 0.332 \\
                      & Ours     & 0.385 & \textbf{0.433} & \underline{0.416} & 0.407 \\
\addlinespace[5pt]
\multicolumn{6}{c}{\textbf{SiT-B}} \\
Metric & Method & CFG 1 & CFG 2 & CFG 2.5 & CFG 3 \\
\midrule
Precision $\uparrow$ & Baseline & 0.766 & \textbf{\underline{0.853}} & 0.848 & 0.840 \\
                      & Ours     & 0.732 & 0.811 & \textbf{\underline{0.824}} & 0.815 \\
Recall $\uparrow$    & Baseline & 0.380 & \underline{0.523} & \textbf{0.525} & 0.506 \\
                      & Ours     & 0.326 & 0.559 & \underline{0.578} & \textbf{0.580} \\
\addlinespace[5pt]
\multicolumn{6}{c}{\textbf{SiT-L}} \\
Metric & Method & CFG 1 & CFG 1.5 & CFG 2 & CFG 2.5 \\
\midrule
Precision $\uparrow$ & Baseline & 0.776 & \underline{0.816} & \textbf{0.822} & 0.813 \\
                      & Ours     & 0.745 & 0.788 & \textbf{\underline{0.798}} & 0.794 \\
Recall $\uparrow$    & Baseline & 0.518 & \underline{0.606} & \textbf{0.616} & 0.604 \\
                      & Ours     & 0.484 & 0.601 & \underline{0.655} & \textbf{0.663} \\
\addlinespace[5pt]
\multicolumn{6}{c}{\textbf{SiT-XL}} \\
Metric & Method & CFG 1 & CFG 1.5 & CFG 1.7 & CFG 2 \\
\midrule
Precision $\uparrow$ & Baseline & 0.776 & \textbf{\underline{0.811}} & 0.809 & 0.808 \\
                      & Ours     & 0.752 & 0.789 & \textbf{\underline{0.795}} & 0.790 \\
Recall $\uparrow$    & Baseline & 0.564 & \underline{0.653} & 0.663 & \textbf{0.671} \\
                      & Ours     & 0.576 & 0.671 & \underline{0.688} & \textbf{0.693} \\
\bottomrule
\end{tabular}
\end{table}

At the CFG selected for each path--flow checkpoint, recall changes from 0.354/0.525/0.616/0.663 to 0.416/0.578/0.655/0.688 for SiT-S/B/L/XL, while precision changes from 0.872/0.848/0.822/0.809 to 0.857/0.824/0.798/0.795. At the lowest tested CFG, path--flow training has lower precision for all four sizes; recall is lower for SiT-B/L and slightly higher for SiT-S/XL.

\paragraph{Multi-seed evaluation.}
\label{app:generation-seed-robustness}
In Table~\ref{tab:generation-seed-robustness} we evaluate each fixed checkpoint on three sets of images generated from three disjoint sets of random seeds. The ranges for the random seeds are 0-49,999, 50k-99,999, 100k-149,999.

\begin{table}[!ht]
\centering
\small
\caption{Mean $\pm$ sample standard deviation of 50k-image FID across three generation-seed ranges, using Heun at 250 NFE.}
\label{tab:generation-seed-robustness}
\begin{tabular}{lcc}
\toprule
Model & Baseline best & Ours best \\
\midrule
SiT-S  & $9.142 \pm 0.13$ & $\mathbf{8.383 \pm 0.041}$ \\
SiT-B  & $5.093 \pm 0.041$ & $\mathbf{4.470 \pm 0.054}$ \\
SiT-L  & $3.458 \pm 0.061$ & $\mathbf{2.777 \pm 0.031}$ \\
SiT-XL & $2.103 \pm 0.026$ & $\mathbf{1.830 \pm 0.015}$ \\
\bottomrule
\end{tabular}
\end{table}

The mean FID is lower for path--flow training at every model size. Each gap is at least $5.8\times$ the larger sample standard deviation observed for that model.

\subsection{Applying our method in more scenarios}

\paragraph{Using our method on FlowDCN.}
\label{app:flowdcn-transfer}
To broaden the applicability of our method, we apply three cycles of path--flow training to a pretrained 120M-parameter FlowDCN-B model~\citep{FlowDCN}, which has a purely convolutional backbone. We select CFG using 10k-FID (dopri5 sampler, see Table~\ref{tab:flowdcn-cfg}), and compare with 25.9k steps of FLOP-matched standard continuation. Final FID is computed from 50k images using the Heun sampler at 250 NFE.

\begin{table}[!ht]
\centering
\small
\caption{FlowDCN-B 10k-image CFG sweep using the 100-step adaptive Dopri5 sampler.}
\label{tab:flowdcn-cfg}
\begin{tabular}{lccc}
\toprule
Checkpoint & CFG 1.5 & CFG 2.0 & CFG 2.5 \\
\midrule
Baseline & 11.890 & \textbf{9.142} & 11.230 \\
Ours     & 15.734 & \textbf{8.204} & 8.765 \\
\bottomrule
\end{tabular}
\end{table}

The sweep selects CFG 2.0 for both checkpoints. At the chosen CFG, Table~\ref{tab:flowdcn-final} compares the performance of the pretrained model, compute-matched continuation, and our method's three-cycle finetuned checkpoint. Path--flow training reaches 5.473 FID, a 7.0\% improvement over the best FLOP-matched continuation result of 5.888, demonstrating that our path--flow alignment method can be applied to various architectures in addition to SiT~\citep{ma2024sit}. 

\begin{table}[!ht]
\centering
\small
\caption{Performance of path--flow alignment using FlowDCN-B (Ours best) compared to continued finetuning (Baseline best), 50k-FID (Heun sampler 250 NFE).}
\label{tab:flowdcn-final}
\begin{tabular}{lccccc}
\toprule
Model & Params & Init FID & Baseline best & Ours best & FID gain \\
\midrule
FlowDCN-B & 120M & 6.489 & 5.888 & \textbf{5.473} & 7.0\% \\
\bottomrule
\end{tabular}
\end{table}

\paragraph{Comparison to~\citep{shankar2025learningstraightflows}  on CIFAR-10.} 
\label{app:cifar-controlled-comparison}
We compare with Learning Straight Flows~\citep{shankar2025learningstraightflows} by initializing both our path--flow alignment method and \citep{shankar2025learningstraightflows}'s method from a SiT-CIFAR/2 checkpoint pretrained for 200k steps. This is a smaller variant of SiT adapted for CIFAR-10 that we pretrained. To ensure fair comparisons, we use matched training FLOPs: our path--flow alignment uses 32.888 PFLOPs, whereas our reproduction of Learning Straight Flows~\citep{shankar2025learningstraightflows} and CIFAR-10 baseline continuation use at least as much FLOPs as was allocated to our method. We compare the methods using 10k-FID at CFG 2.0 and the ODE Euler sampler at 2 and 40 NFE.

Our implementation of Learning Straight Flows~\citep{shankar2025learningstraightflows} includes its empirical target-field estimator, joint flow and straightness objective, and a four-block GLOW/PLU invertible interpolant. We use 100 queries and a balanced 500-image candidate pool. Since we are unable to access the original code and some CIFAR-10 hyperparameters are unavailable, we are unable to reproduce~\citep{shankar2025learningstraightflows} exactly.

\begin{table}[!ht]
\centering
\small
\caption{FLOP-matched CIFAR-10 comparison using 10k images from one generation seed. Lower FID is better.}
\label{tab:cifar-citation17-comparison}
\resizebox{\linewidth}{!}{
\begin{tabular}{lcc}
\toprule
Method & FID, 2 NFE $\downarrow$ & FID, 40 NFE $\downarrow$ \\
\midrule
Our method (1 cycle) & 170.845 & \textbf{9.097} \\
Baseline continuation & 100.518 (step 1,000) & 10.339 (step 1,500) \\
\citep{shankar2025learningstraightflows}, $\lambda_{\max}=0.001$, GLOW LR $5\times10^{-6}$ & \textbf{79.150} & 11.667 \\
\citep{shankar2025learningstraightflows}, $\lambda_{\max}=0.001$, GLOW LR $5\times10^{-7}$ & 84.230 & 11.568 \\
\citep{shankar2025learningstraightflows}, $\lambda_{\max}=0.0003$, GLOW LR $5\times10^{-7}$ & 94.295 & 11.878 \\
\bottomrule
\end{tabular}
}
\end{table}

The standard-continuation row reports the best FID at each NFE achieved during continued finetuning. We report the FID at the end of training for our path--flow alignment and \citep{shankar2025learningstraightflows}. Learning Straight Flows gives the best result at 2 NFE, while path--flow alignment gives the best result at 40 NFE. Intuitively, our method learns a more curved trajectory (see Figure~\ref{fig:toy-learned-trajectories}) whereas \citep{shankar2025learningstraightflows} penalizes the curvature of the resulting trajectory at the cost of high-NFE generation, leading our method to perform better at higher NFE.

\subsection{Qualitative comparisons with our method}

\paragraph{Qualitative samples.}
\label{app:qualitative-samples}
Figure~\ref{fig:sit-xl-qualitative} compares images generated by the pretrained SiT-XL checkpoint with the cycle-3 checkpoint of our path--flow alignment. The grids are uncurated and matched by class and generation seed.

\begin{figure}[!ht]
\centering
\begin{minipage}[t]{0.49\linewidth}
    \centering
    \includegraphics[width=\linewidth]{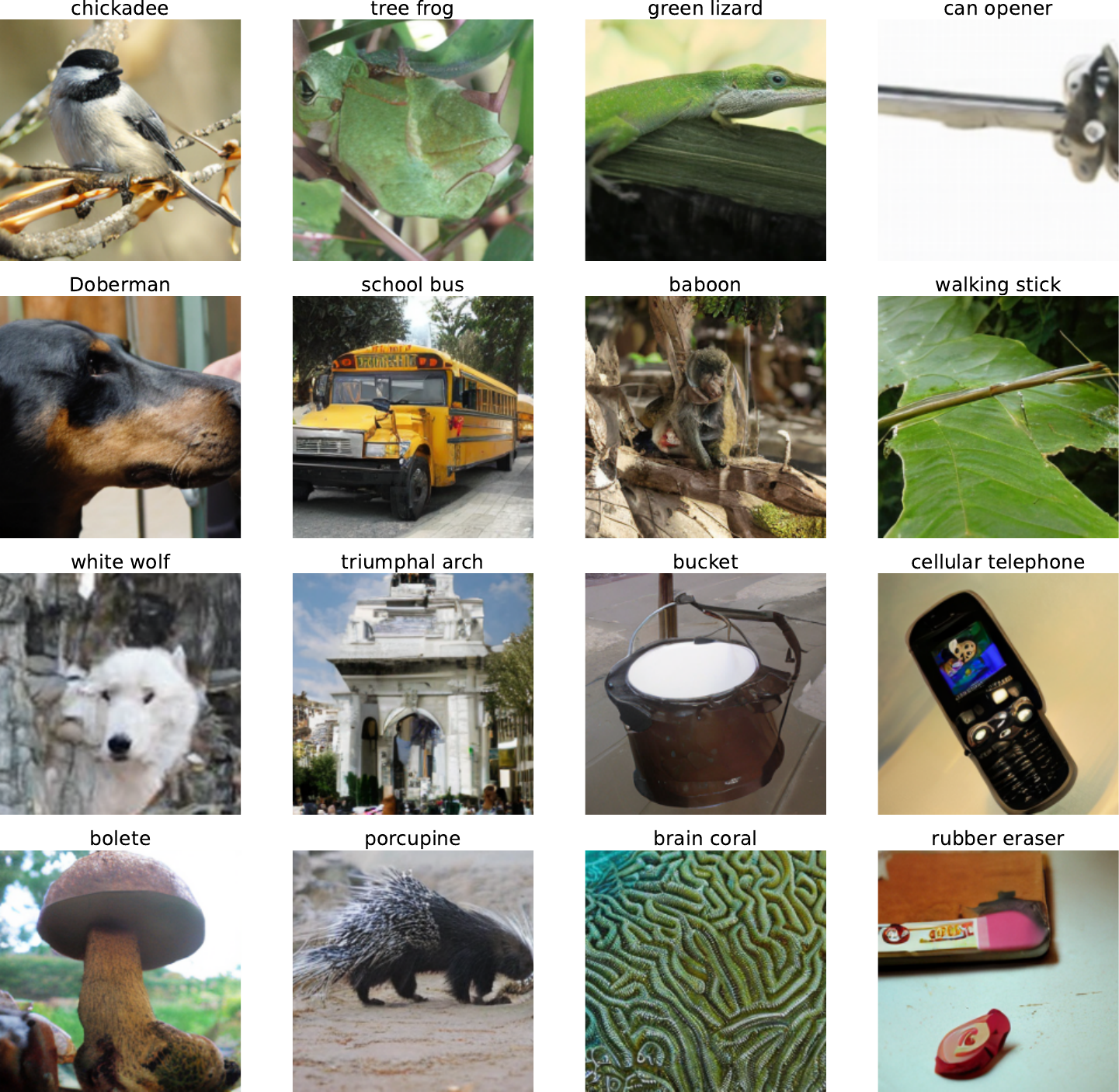}
    \small (a) Best SiT-XL baseline
\end{minipage}
\hfill
\begin{minipage}[t]{0.49\linewidth}
    \centering
    \includegraphics[width=\linewidth]{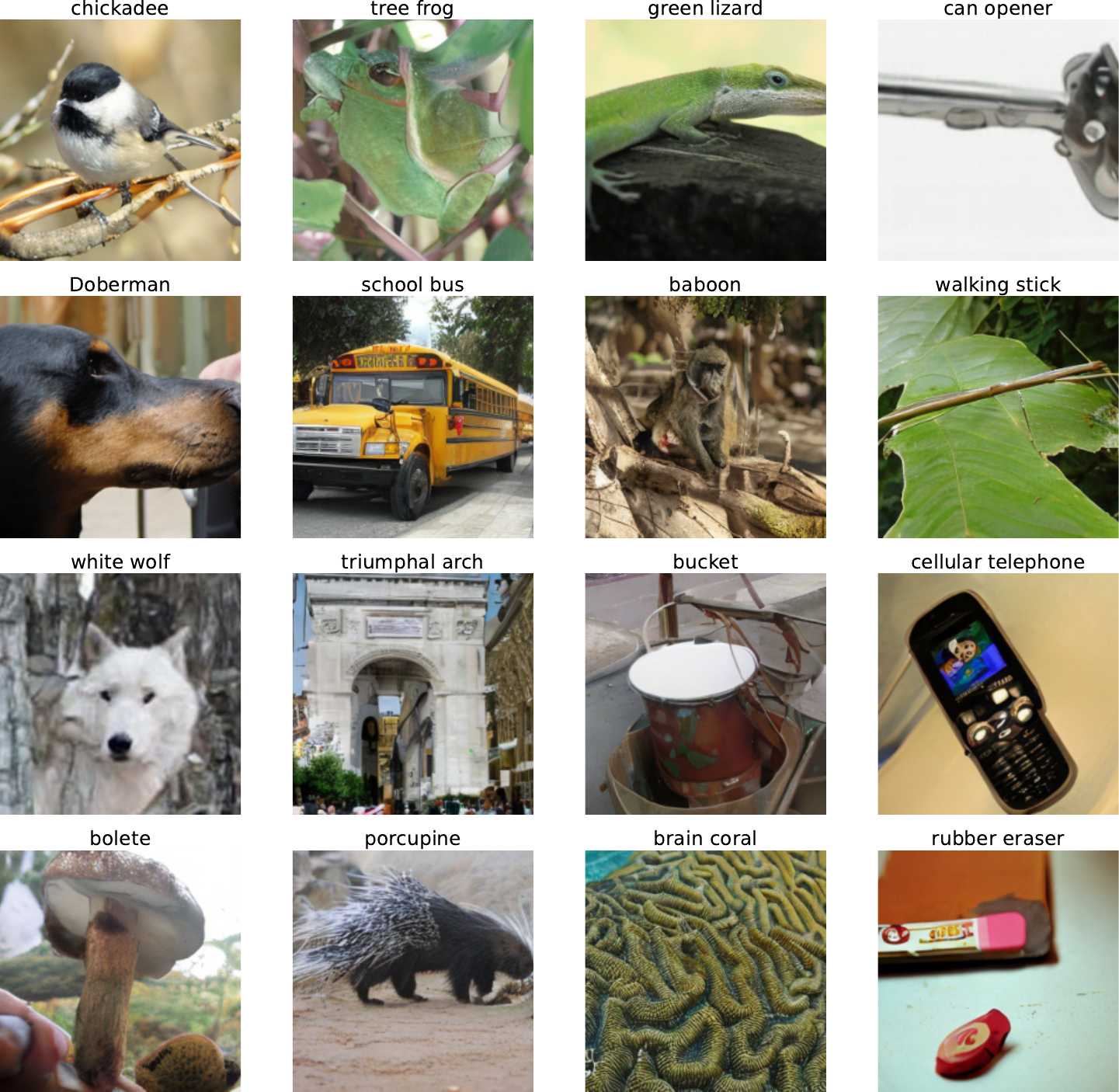}
    \small (b) Ours after three cycles
\end{minipage}
\caption{Uncurated SiT-XL samples matched by class and generation seed. Left: the best baseline, given by the framework initialization. Right: path--flow alignment after three cycles.}
\label{fig:sit-xl-qualitative}
\end{figure}

Most corresponding samples remain similar in class, composition, pose, and background, indicating that our path--flow alignment method refines the pretrained model rather than substantially changing its overall behavior. 

Several samples nevertheless show \textbf{clearer local structure with our method}. At (1,3), the green lizard has more clearly defined feet and contact with the surface. At (2,3), the baboon has a more coherent head and body. At (3,1), the white wolf has a better-formed face. At (3,2), our result is recognizably a triumphal arch, whereas the baseline resembles an irregular building with a small opening. At (4,4), the rubber eraser is cleanly separated from the surrounding material rather than bleeding into it.

\paragraph{Visualizing learned trajectories.}
\label{app:learned-trajectories}
Figure~\ref{fig:toy-learned-trajectories} visualizes the trajectories from a toy experiment under standard flow matching, unregularized path learning, and stochastic path learning. Our method induces a trajectory with higher curvature compared to standard flow matching. 

Our toy example transports an independently coupled two-component Gaussian mixture centered at $(-5,\pm1)$ to one centered at $(5,\pm1)$; each component has isotropic standard deviation $0.3$. We first train a standard flow-matching teacher on linear paths, then initialize the $\rho=1$ and $\rho_t=1-t$ branches from this teacher and run two cycles of alternating path and flow updates. The plotted curves integrate the same 600 source samples with Heun over 240 steps, allowing their geometry to be compared directly.

\ifdefined\arxivversion
\else
\newpage
\input{checklist.tex}
\fi
\end{document}